\documentclass{article}

\PassOptionsToPackage{numbers, compress}{natbib}     
\usepackage[final]{neurips_2020}

\usepackage[utf8]{inputenc} % allow utf-8 input
\usepackage[T1]{fontenc}    % use 8-bit T1 fonts
\usepackage{hyperref}       % hyperlinks
\usepackage{url}            % simple URL typesetting
\usepackage{booktabs}       % professional-quality tables
\usepackage{amsfonts}       % blackboard math symbols
\usepackage{nicefrac}       % compact symbols for 1/2, etc.
\usepackage{microtype}      % microtypography

\usepackage{microtype}
\usepackage{graphicx}
\usepackage{amssymb}
\usepackage{amsmath}
\usepackage{amsthm} % proofs
\usepackage{subfigure}
\usepackage{mathtools}
\usepackage{xcolor} % for colored remarks
\usepackage{tikz} % graph visualizations

\usepackage{caption}
\usepackage{wrapfig}
\usepackage{adjustbox}
\usepackage{hhline}

\newcommand{\diag}{\operatorname{diag}}
\newcommand{\Id}{\operatorname{\boldsymbol{I}}}
\newcommand{\ones}{\operatorname{\boldsymbol{1}}}

\newtheorem{lem}{Lemma}

\title{Scattering GCN: Overcoming Oversmoothness in Graph Convolutional Networks}

\author{%
  Yimeng Min\thanks{Equal contribution; order determined alphabetically} \\
  Mila -- Quebec AI Institute \\
  Montreal, QC, Canada \\
  \texttt{minyimen@mila.quebec} \\
  \And
  Frederik Wenkel\footnotemark[1]\\
  Dept. of Math. and Stat. \\
  Universit\'{e} de Montr\'{e}al \\
  Mila -- Quebec AI Institute \\
  Montreal, QC, Canada \\
  \hspace{-5pt} \texttt{frederik.wenkel@umontreal.ca} \hspace{-5pt} \\
  \And
  Guy Wolf\\
  Dept. of Math. and Stat. \\
  Universit\'{e} de Montr\'{e}al \\
  Mila -- Quebec AI Institute \\
  Montreal, QC, Canada \\
  \texttt{guy.wolf@umontreal.ca} \\
}

\begin{document}

\maketitle

\begin{abstract}
Graph convolutional networks (GCNs) have shown promising results in processing graph data by extracting structure-aware features. This gave rise to extensive work in geometric deep learning, focusing on designing network architectures that ensure neuron activations conform to regularity patterns within the input graph. However, in most cases the graph structure is only accounted for by considering the similarity of activations between adjacent nodes, which limits the capabilities of such methods to discriminate between nodes in a graph. Here, we propose  to augment conventional GCNs with geometric scattering transforms and residual convolutions. The former enables band-pass filtering of graph signals, thus alleviating the so-called oversmoothing often encountered in GCNs, while the latter is introduced to clear the resulting features of high-frequency noise. We establish the advantages of the presented Scattering GCN with both theoretical results establishing the complementary benefits of scattering and GCN features, as well as experimental results showing the benefits of our method compared to leading graph neural networks for semi-supervised node classification, including the recently proposed GAT network that typically alleviates oversmoothing using graph attention mechanisms.
\end{abstract}

\section{Introduction}\label{sect_introduction}
Deep learning approaches are at the forefront of modern machine learning. While they are effective in a multitude of applications, their most impressive results are typically achieved when processing data with inherent structure that can be used to inform the network architecture or the neuron connectivity design. For example, image processing tasks gave rise to convolutional neural networks that rely on spatial organization of pixels, while time series analysis gave rise to recurrent neural networks that leverage temporal organization in their information processing via feedback loops and memory mechanisms. The success of neural networks in such applications, traditionally associated with signal processing, has motivated the emergence of geometric deep learning, with the goal of generalizing the design of structure-aware network architectures from Euclidean spatiotemporal structures to a wide range of non-Euclidean geometries that often underlie modern data.

Geometric deep learning approaches typically use graphs as a model for data geometries, either by constructing them from input data (e.g., via similarity kernels) or directly given as quantified interactions between data points~\citep{bronstein2017geometric}. Using such models, recent works have shown that graph neural networks (GNNs) perform well in multiple application fields, including biology, chemistry and social networks~\citep{gilmer2017neural,hamilton2017inductive,kipf2016semi}. It should be noted that most GNNs consider each graph together with given node features, as a generalization of images or audio signals, and thus aim to compute whole-graph representations. These in turn, can be applied to graph classification, for example when each graph represents the molecular structure of proteins or enzymes classified by their chemical properties~\citep{fout2017protein,de2018molgan,knyazev2018spectral}. 

On the other hand, methods such as graph convolutional networks (GCNs) presented by~\cite{kipf2016semi} consider node-level tasks and in particular node classification.
As explained in~\cite{kipf2016semi}, such tasks are often considered in the context of semi-supervised learning, as typically only a small portion of nodes of the graph possesses labels. In these settings, the entire dataset is considered as one graph and the network is tasked with learning node representations that infer information from node features as well as the graph structure.
However, most state-of-the-art approaches for incorporating graph structure information in neural network operations aim to enforce similarity between representations of adjacent (or neighboring) nodes, which essentially implements local smoothing of neuron activations over the graph~\citep{li2018deeper}.
While such smoothing operations may be sufficiently effective in whole-graph settings, they often cause degradation of results in node processing tasks due to oversmoothing~\citep{li2018deeper,nt2019revisiting}, as nodes become indistinguishable with deeper and increasingly complex network architectures. Graph attention networks \citep{velivckovic2017graph} have shown promising results in overcoming such limitations by introducing adaptive weights for graph smoothing via message passing operations, using attention mechanisms computed from node features and masked by graph edges. However, these networks still essentially rely on enforcing similarity (albeit adaptive) between neighboring nodes, while also requiring more intricate training as their attention mechanism requires gradient computations driven not only by graph nodes, but also by graph edges. We refer the reader to the supplement for further discussion of related work and recent advances in node processing with GNNs.

In this paper, we propose a new approach for node-level processing in GNNs by introducing neural pathways that encode higher-order forms of regularity in graphs. Our construction is inspired by recently proposed geometric scattering networks~\citep{gama2019diffusion,gao2019geometric,zou2019graph}, which have proven effective for whole-graph representation and classification. These networks generalize the Euclidean scattering transform, which was originally presented by~\cite{mallat2012group} as a mathematical model for convolutional neural networks. In graph settings, the scattering construction leverages deep cascades of graph wavelets~\citep{hammond2011wavelets,coifman2006wavelets} and pointwise nonlinearities to capture multiple modes of variation from node features or labels. Using the terminology of graph signal processing, these can be considered as generalized band-pass filtering operations, while GCNs (and many other GNNs) can be considered as relying on low-pass filters only.
Our approach combines together the merits of GCNs on node-level tasks with those of scattering networks known from whole-graph tasks, by enabling learned node-level features to encode geometric information beyond smoothed activation signals, thus alleviating oversmoothing concerns often raised in GCN approaches. We discuss the benefits of our approach and demonstrate its advantages over GCNs and other popular graph processing approaches for semi-supervised node classification, including significant improvements on the DBLP graph dataset from~\cite{pan2016tri}.
\paragraph{Notations:}
We denote matrices and vectors with bold letters with uppercase letters representing matrices and lowercase letters representing vectors. In particular, $\Id_n\in\mathbb{R}^{n\times n}$ is used for the identity matrix and $\ones_n\in \mathbb{R}^n$ denotes the vector with ones in every component. We write $\langle.,.\rangle$ for the standard scalar product in $\mathbb{R}^n$. We will interchangeably consider functions of graph nodes as vectors indexed by the nodes, implicitly assuming a correspondence between a node and a specific index. This carries over to matrices, where we relate nodes to column or row indices. We further use the abbreviation $[n] \coloneqq \{1,\dots,n\}$ where $n\in\mathbb{N}$ and write $\mathbb{N}_0\coloneqq \mathbb{N}\cup\{0\}$.

\section{Graph Signal Processing}\label{sect_gsp}
Let $G = (V,E,w)$ be a weighted graph with $V\coloneqq \{v_1,\dots,v_n\}$ the set of nodes, $E\subset \{\{v_i, v_j\}\in V\times V , i\neq j\}$ the set of (undirected) edges and $w : E \to (0,\infty)$ assigning (positive) edge weights to the graph edges. We note that $w$ can equivalently be considered as a function of $V \times V$, where we set the weights of non-adjacent node pairs to zero. We define a \textit{graph signal} as a function $x: V \rightarrow \mathbb{R}$ on the nodes of $G$ and aggregate them in a signal vector $\boldsymbol{x}\in \mathbb{R}^n$ with the $i^{th}$ entry being $x(v_i)$.

We define the (combinatorial) \textit{graph Laplacian} matrix $\boldsymbol{L}\coloneqq \boldsymbol{D} - \boldsymbol{W}$, where $\boldsymbol{W}\in\mathbb{R}^{n\times n}$ is the \textit{weighted adjacency matrix} of the graph $G$ given by
\[
    \boldsymbol{W}[v_i,v_j] \coloneqq
    \begin{cases} 
    w(v_i,v_j) & \text{if } \{v_i,v_j\}\in E \\
    0 & \text{otherwise}
    \end{cases},
\]
and $\boldsymbol{D}\in\mathbb{R}^{n\times n}$ is the \textit{degree matrix} of $G$ defined by $\boldsymbol{D}\coloneqq \diag(d_1,\dots, d_n)$ with $d_i\coloneqq \deg(v_i)\coloneqq \sum_{j=1}^n \boldsymbol{W}[v_i,v_j]$ being the \textit{degree} of the node $v_i$. In practice, we work with the (symmetric) \textit{normalized Laplacian} matrix
$
    \boldsymbol{\mathcal{L}} \coloneqq \boldsymbol{D}^{-1/2}\boldsymbol{L}\boldsymbol{D}^{-1/2} = \Id_n - \boldsymbol{D}^{-1/2}\boldsymbol{W}\boldsymbol{D}^{-1/2}.
$
It can be verified that $\boldsymbol{\mathcal{L}}$ is symmetric and positive semi-definite and can thus be orthogonally diagonalized as
$
    \boldsymbol{\mathcal{L}} = \boldsymbol{Q}\boldsymbol{\Lambda}\boldsymbol{Q}^T = \sum_{i=1}^n \lambda_i \boldsymbol{q}_i \boldsymbol{q}_i^T,
$
where $\boldsymbol{\Lambda}\coloneqq \diag(\lambda_1,\dots,\lambda_n)$ is a diagonal matrix with the eigenvalues on the main diagonal and $\boldsymbol{Q}$ is an orthogonal matrix containing the corresponding normalized eigenvectors $\boldsymbol{q}_1,\dots,\boldsymbol{q}_n\in\mathbb{R}^n$ as its columns. 

A detailed study (see, e.g.,~\cite{Chung:1997}) of the eigenvalues reveals that $0 = \lambda_1 \leqslant \lambda_2 \leqslant \dots \leqslant \lambda_n \leqslant 2$. We can interpret the $\lambda_i, i\in[n]$ as the frequency magnitudes and the $\boldsymbol{q}_i$ as the corresponding Fourier modes. We accordingly define the \textit{Fourier transform} of a signal vector $\boldsymbol{x}\in \mathbb{R}^n$ by $\boldsymbol{\hat x}[i] = \langle \boldsymbol{x}, \boldsymbol{q}_i \rangle$ for $i\in[n]$. The corresponding inverse Fourier transform is given by $\boldsymbol{x} = \sum_{i=1}^n \boldsymbol{\hat x}[i] \boldsymbol{q}_i$. Note that this can be written compactly as $\boldsymbol{\hat x} = \boldsymbol{Q}^T \boldsymbol{x}$ and $\boldsymbol{x} = \boldsymbol{Q} \boldsymbol{\hat x}$. Finally, we introduce the concept of \textit{graph convolutions}. We define a filter $g: V\rightarrow \mathbb{R}$ defined on the set of nodes and want to convolve the corresponding filter vector $\boldsymbol{g}\in\mathbb{R}^n$ with a signal vector $\boldsymbol{x}\in\mathbb{R}^n$, i.e. $\boldsymbol{g}\star \boldsymbol{x}$. To explicitly compute this convolution, we recall that in the Euclidean setting, the convolution of two signals equals the product of their corresponding frequencies. This property generalizes to graphs \citep{shuman2016vertex} in the sense that $(\widehat{\boldsymbol{g}\star \boldsymbol{x}})[i] = \boldsymbol{\hat g}[i]\boldsymbol{\hat x}[i]$ for $i\in[n]$. Applying the inverse Fourier transform yields
$$
    \boldsymbol{g}\star \boldsymbol{x}
    = \sum_{i=1}^n \boldsymbol{\hat g}[i] \boldsymbol{\hat x}[i] \boldsymbol{q}_i
    = \sum_{i=1}^n \boldsymbol{\hat g}[i] \langle \boldsymbol{q}_i, \boldsymbol{x} \rangle \boldsymbol{q}_i
    = \boldsymbol{Q} \boldsymbol{\widehat{G}} \boldsymbol{Q}^T \boldsymbol{x},
$$
where $\boldsymbol{\widehat{G}} \coloneqq \diag(\boldsymbol{\hat g}) = \diag(\boldsymbol{\hat g}[1], \dots, \boldsymbol{\hat g}[n])$. Hence, convolutional graph filters can be parameterized by considering the Fourier coefficients in $\boldsymbol{\widehat{G}}$. 

Furthermore, it can be verified~\citep{defferrard2016convolutional} that when these coefficients are defined as polynomials $\boldsymbol{\hat g}[i] \coloneqq \sum_k \gamma_k \lambda_i^k$ for $i\in \mathbb{N}$ of the Laplacian eigenvalues in $\boldsymbol{\Lambda}$ (i.e. $\boldsymbol{\widehat{G}} = \sum_k \gamma_k \boldsymbol{\Lambda}^k$), the resulting filter convolution are localized in space and can be written in terms of $\mathcal{L}$ as
$
    \boldsymbol{g}\star \boldsymbol{x} = \sum_k \gamma_k \mathcal{L}^k \boldsymbol{x}
$
without requiring spectral decomposition of the normalized Laplacian. This motivates the standard practice~\citep{kipf2016semi,defferrard2016convolutional,susnjara2015accelerated,liao2019lanczosnet} of using filters that have polynomial forms, which we follow here as well. 

For completeness, we note there exist alternative frameworks that generalize signal processing notions to graph domains, such as~\cite{oyallon2020interferometric}, which emphasizes the construction of complex filters that requires a notion of signal phase on graphs. However, extensive study of such alternatives is out of scope for the current work, which thus relies on the well-established (see, e.g.,~\cite{shuman2013emerging}) framework described here.
\section{Graph Convolutional Network}\label{sect_graph conv. net}
Graph convolutional networks (GCNs), introduced in~\cite{kipf2016semi}, consider semi-supervised settings where only a small potion of the nodes is labeled. They leverage intrinsic geometric information encoded in the adjacency matrix $\boldsymbol{W}$ together with node labels by constructing a convolutional filter parametrized by $\boldsymbol{\hat{g}}[i] \coloneqq \theta (2 - \lambda_i)$, where the choice of a single learnable parameter is made to avoid overfitting. This parametrization yields a convolutional filtering operation given by
\begin{equation}\label{eq_conv. parametrization}
    \boldsymbol{g}_\theta \star \boldsymbol{x} = \theta \left(\Id_n + \boldsymbol{D}^{-1/2} \boldsymbol{W} \boldsymbol{D}^{-1/2}\right) \boldsymbol{x}.
\end{equation}
The matrix $\Id_n + \boldsymbol{D}^{-1/2} \boldsymbol{W} \boldsymbol{D}^{-1/2}$ has eigenvalues in $[0,2]$. This could lead to vanishing or exploding gradients. This issue is addressed by the following renormalization trick~\citep{kipf2016semi}: $\Id_n + \boldsymbol{D}^{-1/2} \boldsymbol{W} \boldsymbol{D}^{-1/2} \rightarrow \boldsymbol{\tilde D}^{-1/2} \boldsymbol{\tilde W} \boldsymbol{\tilde D}^{-1/2}$, where $\boldsymbol{\tilde W}\coloneqq \Id_n + \boldsymbol{W}$ and $\boldsymbol{\tilde D}$ a diagonal matrix with $\boldsymbol{\tilde D}[v_i,v_i] \coloneqq \sum_{j=1}^n \boldsymbol{\tilde W}[v_i,v_j]$ for $i\in[n]$. This operation replaces the features of the nodes by a weighted average of itself and its neighbors. Note that the repeated execution of graph convolutions will enforce similarity throughout higher-order neighborhoods with order equal to the number of stacked layers. Setting
$
    \boldsymbol{A} \coloneqq \boldsymbol{\tilde D}^{-1/2} \boldsymbol{\tilde W} \boldsymbol{\tilde D}^{-1/2},
$
the complete layer-wise propagation rule takes the form
$
 \boldsymbol{h}_j^\ell = \sigma \big( \sum_{i=1}^{N_{\ell-1}} \theta_{ij}^\ell \boldsymbol{A} \boldsymbol{h}_i^{\ell-1} \big),
$
where $\ell$ indicates the layer with $N_\ell$ neurons, $\boldsymbol{h}_j^\ell \in \mathbb{R}^n$ the activation vector of the $j^{th}$ neuron, $\theta_{ij}^\ell$ the learned parameter of the convolution with the $i^{th}$ incoming activation vector from the preceding layer and $\sigma(.)$ an element-wise applied activation function. Written in matrix notation, this gives
\begin{equation}\label{eq_gcn}
    \boldsymbol{H}^\ell = \sigma \left( \boldsymbol{A} \boldsymbol{H}^{\ell-1} \boldsymbol{\Theta}^\ell \right),
\end{equation}
where $\boldsymbol{\Theta}^\ell\in\mathbb{R}^{N_{\ell-1}\times N_\ell}$ is the weight-matrix of the $\ell^{th}$ layer and $\boldsymbol{H}^\ell\in\mathbb{R}^{n\times N_\ell}$ contains the activations outputted by the $\ell^{th}$ layer.
 
We remark that the above explained GCN model can be interpreted as a low-pass operation. For the sake of simplicity, let us consider the convolutional operation (Eq.~\ref{eq_conv. parametrization}) before the reparametrization trick. If we observe the convolution operation as the summation
$
    \boldsymbol{g}_\theta \star \boldsymbol{x} = \sum_{i=1}^n \boldsymbol{\gamma}_i \boldsymbol{\hat x}[i] \boldsymbol{q}_i,
$
we clearly see that higher weights $\boldsymbol{\gamma}_i = \theta(2 - \lambda_i)$ are put on the low-frequency harmonics, while high-frequency harmonics are progressively less involved  as $0 = \lambda_1 \leqslant \lambda_2 \leqslant \dots \leqslant \lambda_n \leqslant 2$. This indicates that the model can only access a diminishing portion of the original information contained in the input signal the more graph convolutions are stacked. This observation is in line with the well-known oversmoothing problem~\citep{li2018deeper} related to GCN models. The repeated application of graph convolutions will successively smooth the signals of the graph such that nodes cannot be distinguished anymore.
\section{Geometric Scattering}\label{sect_scattering}
\label{sect_geoscat}
In this section, we recall the construction of geometric scattering on graphs. This construction is based on the \textit{lazy random walk} matrix
$$
    \boldsymbol{P} \coloneqq \frac{1}{2} \big( \Id_n + \boldsymbol{W} \boldsymbol{D}^{-1} \big),
$$
which is closely related to the \textit{graph random walk} defined as a Markov process with transition matrix $\boldsymbol{R}\coloneqq \boldsymbol{W}\boldsymbol{D}^{-1}$. The matrix $\boldsymbol{P}$ however allows self loops while normalizing by a factor of two in order to retain a Markov process.
Therefore, considering a distribution $\boldsymbol{\mu}_0 \in \mathbb{R}^n$ of the initial  position of the lazy random walk, its positional distribution after $t$ steps is encoded by $\boldsymbol{\mu}_t = \boldsymbol{P}^t \boldsymbol{\mu}_0$.

As discussed in~\citep{gao2019geometric}, the propagation of a graph signal vector $\boldsymbol{x}\in\mathbb{R}^n$ by $\boldsymbol{x}_t = \boldsymbol{P}^t \boldsymbol{x}$ performs a low-pass operation that preserves the zero-frequencies of the signal while suppressing high frequencies. In geometric scattering, this low-pass information is augmented by introducing the \textit{wavelet} matrices $\boldsymbol{\Psi}_k\in\mathbb{R}^{n\times n}$ of scale $2^k$, $k\in \mathbb{N}_0$,
\begin{equation}\label{eq_wavelet matrix}
    \begin{cases}
    \boldsymbol{\Psi}_0 \coloneqq \Id_n - \boldsymbol{P}, \\
    
    \boldsymbol{\Psi}_k \coloneqq \boldsymbol{P}^{2^{k-1}} - \boldsymbol{P}^{2^k} = \boldsymbol{P}^{2^{k-1}} \big( \Id_n - \boldsymbol{P}^{2^{k-1}} \big), \quad k\geq 1.
\end{cases}
\end{equation}
This leverages the fact that high frequencies can be recovered with multiscale wavelet transforms, e.g., by decomposing nonzero frequencies into dyadic frequency bands. The operation $(\boldsymbol{\Psi}_k \boldsymbol{x})[v_i]$ collects signals from a neighborhood of order $2^k$, but extracts multiscale differences rather than averaging over them. The wavelets in Eq.~\ref{eq_wavelet matrix} can be organized in a filter bank $\{\boldsymbol{\Psi}_k, \boldsymbol{\Phi}_K\}_{0\leq k\leq K}$, where $\boldsymbol{\Phi}_K\coloneqq \boldsymbol{P}^{2^K}$ is a pure low-pass filter. The telescoping sum of the matrices in this filter bank constitutes the identity matrix, thus enabling to reconstruct processed signals from their filter responses. Further studies of this construction and its properties (e.g., energy preservation) appear in~\cite{perlmutter2019understanding} and related work.

\textit{Geometric scattering} was originally introduced in the context of whole-graph classification and consisted of aggregating \textit{scattering features}. These are stacked wavelet transforms (see Fig.~\ref{fig:GeoSct_illustration}) parameterized via tuples $p \coloneqq (k_1, \dots, k_m)\in \cup_{m \in \mathbb{N}} \mathbb{N}_0^{m}$ containing the bandwidth scale parameters, which are separated by element-wise absolute value nonlinearities\footnote{In a slight deviation from previous work, here $\boldsymbol{U}_p$ does not include the outermost nonlinearity in the cascade.} according to
\begin{equation}
    \label{eq_original node.scattering}
     \boldsymbol{U}_p \boldsymbol{x} \coloneqq \boldsymbol{\Psi}_{k_m} \vert \boldsymbol{\Psi}_{k_{m-1}} \dots \vert \boldsymbol{\Psi}_{k_2} \vert \boldsymbol{\Psi}_{k_1}\boldsymbol{x}\vert \vert \dots \vert,
\end{equation}
where $m$ corresponds to the length of the tuple $p$. The scattering features are aggregated over the whole graph by taking $q^{th}$-order moments over the set of nodes,
\begin{equation}\label{eq_original geom.scattering} \textstyle
    \boldsymbol{S}_{p,q} \boldsymbol{x} \coloneqq \sum_{i=1}^n \vert \boldsymbol{U}_p \boldsymbol{x} [v_i] \vert^q.
\end{equation}

\begin{figure}
\centering
\includegraphics[width=0.6\textwidth]{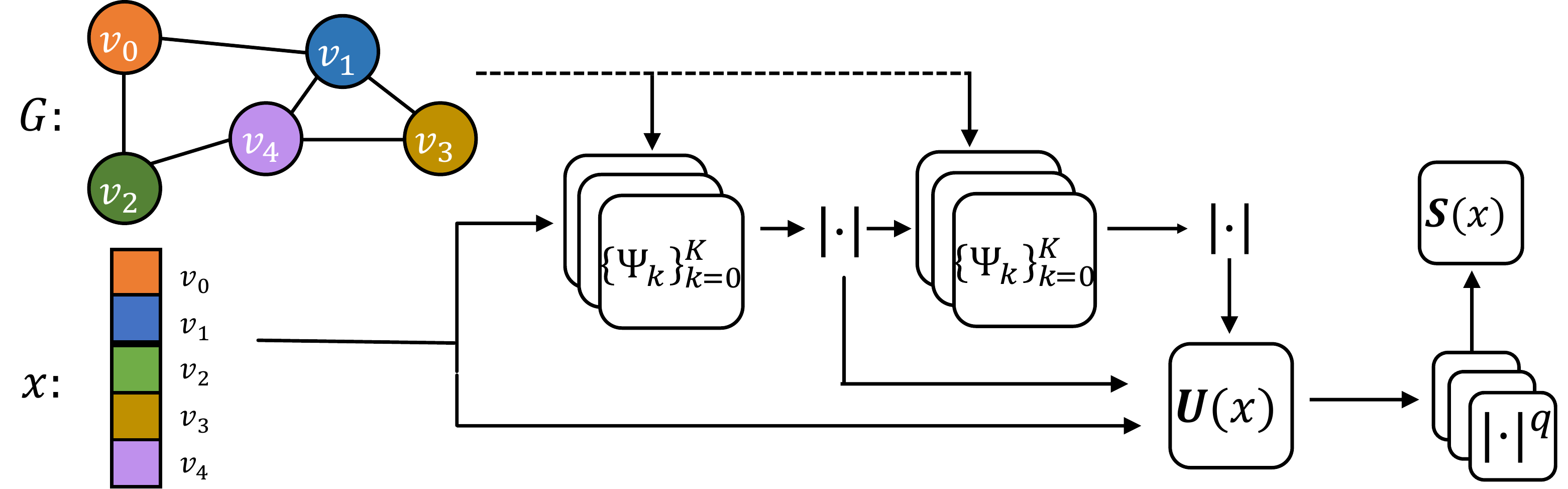}
\caption{Illustration of geom. scattering at the node level ($\boldsymbol{U}(\boldsymbol{x}) = \{\boldsymbol{U}_p \boldsymbol{x} : p \in \mathbb{N}_0^{m}, m = 0,1,2\}$) and at the graph level ($\boldsymbol{S}(\boldsymbol{x}) = \{\boldsymbol{S}_{p,q} \boldsymbol{x} : q \in \mathbb{N}, p \in \mathbb{N}_0^{m}, m = 0,1,2\}$), extracted according to the wavelet cascade in Eqs.~\ref{eq_wavelet matrix}-\ref{eq_original geom.scattering}. While $m \leq 2$ orders are illustrated here, more can be used in general.}
\label{fig:GeoSct_illustration}
\end{figure}

As our work is devoted to the study of node-based classification, we reinvent this approach in a new context, keeping the scattering transforms $\boldsymbol{U}_p$ on a node-level by dismissing the aggregation step in Eq. \ref{eq_original geom.scattering}. 
For each tuple $p$, we define the following scattering propagation rule, which mirrors the GCN rule but replaces the low-pass filter by a geometric scattering operation resulting in
\begin{equation}\label{eq_sct}
    \boldsymbol{H}^\ell = \sigma \left( \boldsymbol{U}_{p} \boldsymbol{H}^{\ell-1} \boldsymbol{\Theta}^\ell\right).
\end{equation}
We note that in practice, we only choose a subset of tuples, which is chosen as part of the network design explained in the following section.

\section{Combining GCN and Scattering Models}\label{sect_combining methods}
To combine the benefits of GCN models and geometric scattering adapted to the node level, we now propose a hybrid network architecture as shown in Fig.~\ref{fig:1}. It combines low-pass operations based on GCN models with band-pass operations based on geometric scattering. To define the layer-wise propagation rule, we introduce
\begin{equation*}
    \boldsymbol{H}_{gcn}^\ell \coloneqq \left[ \boldsymbol{H}_{gcn,1}^\ell \mathbin\Vert \dots \mathbin\Vert \boldsymbol{H}_{gcn,C_{gcn}}^\ell \right]
\quad \text{ and } \quad
    \boldsymbol{H}_{sct}^\ell \coloneqq \left[ \boldsymbol{H}_{sct,1}^\ell \mathbin\Vert \dots \mathbin\Vert \boldsymbol{H}_{sct,C_{sct}}^\ell \right],
\end{equation*}
which are the concatenations of channels $\big\{ \boldsymbol{H}_{gcn,k}^\ell \big\}_{k=1}^{C_{gcn}}$ and $\big\{ \boldsymbol{H}_{sct,k}^\ell \big\}_{k=1}^{C_{sct}}$, respectively. Every $\boldsymbol{H}_{gcn,k}^\ell$ is defined according to Eq.~\ref{eq_gcn} with the slight modification of added biases and powers of $\boldsymbol{A}$,
\begin{figure}
\centering
\includegraphics[width=0.77\textwidth]{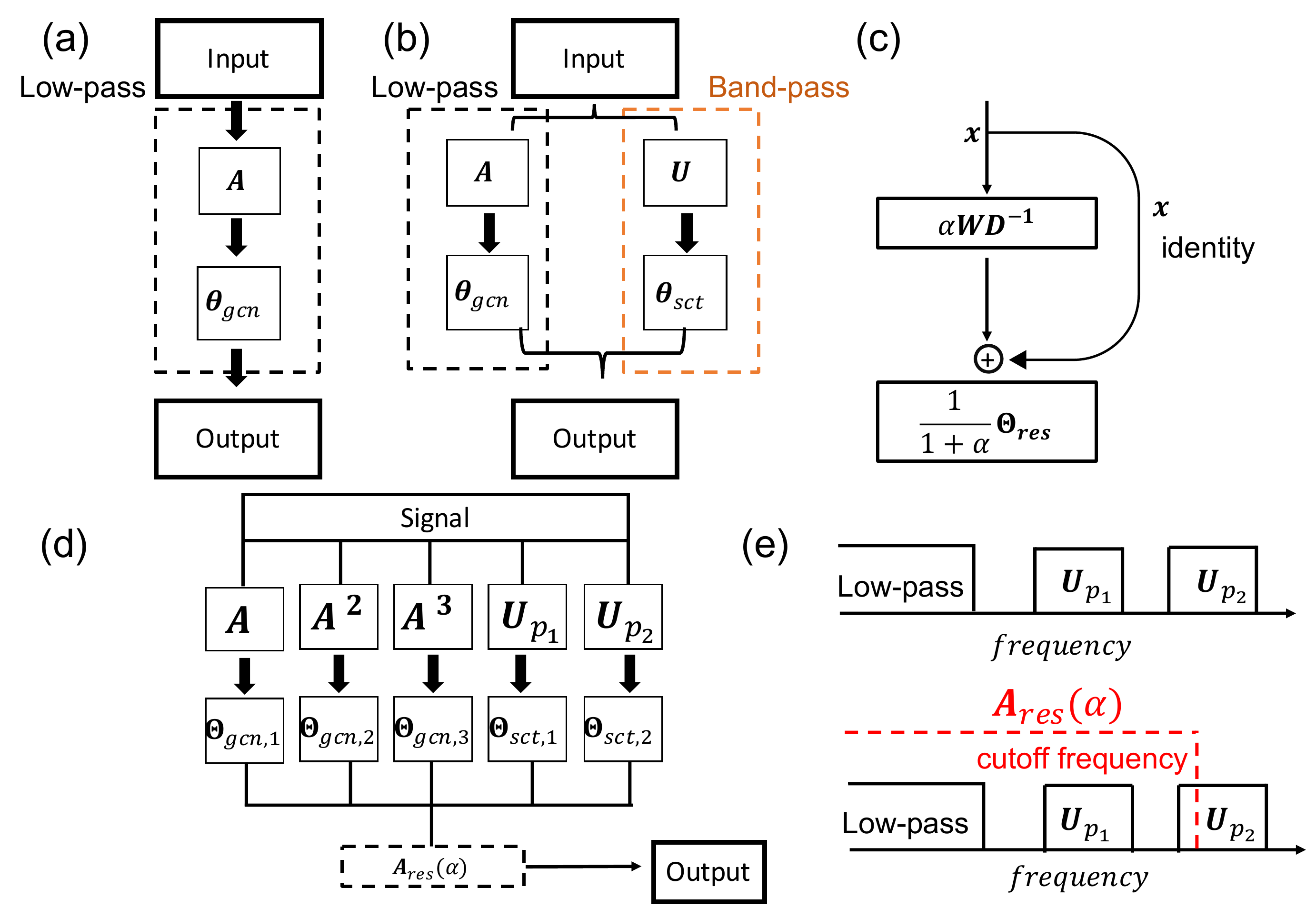}
\caption{(a,b) Comparison between GCN and our network: we add band-pass channels to collect different frequency components; (c) Graph residual convolution layer; (d) Band-pass layers; (e) Schematic depiction in the frequency domain. }
\label{fig:1}
\end{figure}
$$
    \boldsymbol{H}_{gcn,k}^\ell \coloneqq \sigma \left( \boldsymbol{A}^k \boldsymbol{H}^{\ell-1} \boldsymbol{\Theta}_{gcn,k}^\ell + \boldsymbol{B}_{gcn,k}^\ell \right).
$$

Note that every GCN filter uses a different propagation matrix $\boldsymbol{A}^k$ and therefore aggregates information from $k$-step neighborhoods. Similarly, we proceed with $\boldsymbol{H}_{sct,k}^\ell$ according to Eq.~\ref{eq_sct} and calculate
\begin{equation*}
    \boldsymbol{H}_{sct,k}^\ell \coloneqq \sigma \left( \boldsymbol{U}_{p_k} \boldsymbol{H}^{\ell-1} \boldsymbol{\Theta}_{sct,k}^\ell + \boldsymbol{B}_{sct,k}^\ell \right), %\in \mathbb{R}^{n\times N_\ell^{sct,k}}
\end{equation*}
where $p_k \in \bigcup_{m\in\mathbb{N}} \mathbb{N}_0^m$, $k=1,\ldots,C_{sct}$ enables scatterings of different orders and scales. Finally, the GCN components and scattering components get concatenated to
\begin{equation}\label{eq_concatenated layer}
    \boldsymbol{H}^\ell \coloneqq \left[ \boldsymbol{H}_{gcn}^\ell \mathbin\Vert \boldsymbol{H}_{sct}^\ell \right].
\end{equation}

The learned parameters are the weight matrices $\boldsymbol{\Theta}_{gcn,k}^\ell, \boldsymbol{\Theta}_{sct,k}^\ell \in\mathbb{R}^{N_{\ell-1}\times N_\ell}$ coming from the convolutional and scattering layers. These are complemented by vectors of the biases $\boldsymbol{b}_{gcn,k}^\ell, \boldsymbol{b}_{sct,k}^\ell \in \mathbb{R}^{N_\ell}$, which are transposed and vertically concatenated $n$ times to the matrices $\boldsymbol{B}_{gcn,k}, \boldsymbol{B}_{sct,k} \in \mathbb{R}^{n\times N_{\ell}}$. To simplify notation, we assume here that all channels use the same number of neurons ($N_\ell$). Waiving this assumption would slightly complicate the notation but works perfectly fine in practice.

In this work, for simplicity, and because it is sufficient to establish our claim, we limit our architecture to three GCN channels and two scattering channels as illustrated in Fig.~\ref{fig:1} (b). Inspired by the aggregation step in classical geometric scattering, we use $\sigma(.) \coloneqq \vert~.~\vert^q$ as our nonlinearity. However, unlike the powers in Eq.~\ref{eq_original geom.scattering}, the $q^{th}$ power is applied at the node-level here instead of being aggregated as moments over the entire graph, thus retaining the distinction between node-wise activations. 

We set the input of the first layer $\boldsymbol{H}^{0}$ to have the original node features as the graph signal. Each subchannel (GCN or scattering) transforms the original feature space to a new hidden space with the dimension determined by the number of neurons encoded in the columns of the corresponding submatrix of $\boldsymbol{H}^{\ell}$. These transformations are learned by the network via the weights and biases. Larger matrices $\boldsymbol{H}^{\ell}$ (i.e., more columns as the number of nodes in the graph is fixed) indicate that the weight matrices have more parameters to learn. Thus, the information in these channels can be propagated well and will be sufficiently represented. 

In general, the width of a channel is relevant for the importance of the captured regularities. A wider channel suggests that these frequency components are more critical and need to be sufficiently learned. Reducing the width of the channel suppresses the magnitude of information that can be learned from a particular frequency window. For more details and analysis of specific design choices in our architecture we refer the reader to the ablation study provided in the supplement.
\section{Graph Residual Convolution}\label{sect_graph res. conv.} 
Using the combination of GCN and scattering architectures, we collect multiscale information at the node level. This information is aggregated from different localized neighborhoods, which may exhibit vastly different frequency spectra. This comes for example from varying label rates in different graph substructures. In particular, very sparse graph sections can cause problems when the scattering features actually learn the difference between labeled and unlabeled nodes, creating high-frequency noise. In the classical geometric scattering used for whole-graph representation, geometric moments were used to aggregate the node-based information, serving at the same time as a low-pass filter. As we want to keep the information localized on the node level, we choose a different approach inspired by skip connections in residual neural networks~\citep{he2016deep}. Conceptually, this low-pass filter, which we call \textit{graph residual convolution}, reduces the captured frequency spectrum up to a cutoff frequency as depicted in Fig.~\ref{fig:1} (e).

The graph residual convolution matrix, governed by the hyperparameter $\alpha$, is given by
$ %\begin{equation*}
    \boldsymbol{A}_{res}(\alpha) = \frac{1}{\alpha+1} (\boldsymbol{I}_n+\alpha \boldsymbol{W} \boldsymbol{D}^{-1})
$ %\end{equation*}
and we apply it after the hybrid layer of GCN and scattering filters. For $\alpha = 0$ we get the identity (no cutoff), while $\alpha \rightarrow \infty$ results in $\boldsymbol{R} = \boldsymbol{W} \boldsymbol{D}^{-1}$. This can be interpreted as an interpolation between the completely lazy (i.e., stationary) random walk and the non-resting (i.e., with no self-loops) random walk $\boldsymbol{R}$.
We apply the graph residual layer on the output $\boldsymbol{H}^\ell$ of the Scattering GCN layer (Eq.~\ref{eq_concatenated layer}). The update rule for this step, illustrated in Fig.~\ref{fig:1} (c), is then expressed by
$ %\begin{equation*}
\boldsymbol{H}^{\ell+1} = \boldsymbol{A}_{res}(\alpha) \boldsymbol{H}^\ell \boldsymbol{ \Theta}_{res} + \boldsymbol{B}_{res},
$ %\end{equation*}
where $\boldsymbol{\Theta}_{res}$ $\in$ $\mathbb{R}^{N\times N_{\ell+1}}$ are learned weights, $\boldsymbol{B}_{res}$ $\in \mathbb{R}^{n \times N_{\ell+1}}$ are learned biases (similar to the notations used previously), and $N$ is the number of features of the concatenated layer $\boldsymbol{H}^\ell$. If $\boldsymbol{H}^{\ell+1}$ is the final layer, we choose $N_{\ell+1}$ equal to the number of classes.
\section{Additional Information Introduced by Node-level Scattering Features}\label{sect_theory}
Before empirically verifying the viability of the proposed architecture in node classification tasks, we first discuss and demonstrate the additional information provided by scattering channels beyond that provided by traditional GCN channels. We first consider information carried by node features, treated as graph signals, and in particular their regularity over the graph. As discussed in Sec.~\ref{sect_graph conv. net}, such regularity is traditionally considered only via smoothness of signals over the graph, as only low frequencies are retained by (local) smoothing operations.
Band-pass filtering, on the other hand, can retain other forms of regularity such as periodic or harmonic patterns. The following lemma demonstrates this difference between GCN and scattering channels.

\begin{lem}\label{thm_Lemma 1}
Consider a cyclic graph on $2n$ nodes, $n\in\mathbb{N}$, and let $\boldsymbol{x}\in\mathbb{R}^{2n}$ be a 2-periodic signal on it (i.e., $\boldsymbol{x}_{2\ell-1}=a$ and $\boldsymbol{x}_{2\ell}=b$, for $\ell\in[n]$ for some $a \neq b \in \mathbb{R}$). Then, for any $\theta \in \mathbb{R}$, the GCN filtering $\boldsymbol{g}_\theta \star \boldsymbol{x}$ from Eq.~\ref{eq_conv. parametrization} yields a constant signal, while the scattering filter $\boldsymbol{\Psi}_0 \boldsymbol{x}$ from Eq.~\ref{eq_wavelet matrix} still produces a 2-periodic signal. Further, this result extends to any finite linear cascade of such filters (i.e., $\boldsymbol{g}_\theta \star \cdots \star \boldsymbol{g}_\theta \star \boldsymbol{x}$ or $\boldsymbol{\Psi}_0 \cdots \boldsymbol{\Psi}_0 \boldsymbol{x}$ with $k\in\mathbb{N}$ filter applications in each).
\end{lem}
While this is only a simple example, it already indicates a fundamental difference between the regularity patterns considered in graph convolutions compared to our approach. Indeed, it implies that if a smoothing convolutional filter encounters alternating signals on isolated cyclic substructures within a graph, their node features become indistinguishable, while scattering channels (with appropriate scales, weights and bias terms) will be able to make this distinction. Moreover, this difference can be generalized further beyond cyclic structures to consider features encoding two-coloring information on constant-degree bipartite graphs, as shown in the following lemma. We refer the reader to the supplement for a proof of this lemma, which also covers the previous one as a particular case, as well as numerical examples illustrating the results in these two lemmas.
\begin{lem}\label{thm_Lemma 2}
Consider a bipartite graph on $n\in\mathbb{N}$ nodes with constant node degree $\beta$. Let $\boldsymbol{x}\in\mathbb{R}^{n}$ be a 2-coloring signal (i.e., with one part assigned constant $a$ and the other $b$, for some $a \neq b \in \mathbb{R}$). Then, for any $\theta \in \mathbb{R}$, the GCN filtering $\boldsymbol{g}_\theta \star \boldsymbol{x}$ from Eq.~\ref{eq_conv. parametrization} yields a constant signal, while the scattering filter $\boldsymbol{\Psi}_0 \boldsymbol{x}$ from Eq.~\ref{eq_wavelet matrix} still produces a (non-constant) 2-coloring of the graph. Further, this result extends to any finite linear cascade of such filters (i.e., $\boldsymbol{g}_\theta \star \cdots \star \boldsymbol{g}_\theta \star \boldsymbol{x}$ or $\boldsymbol{\Psi}_0 \cdots \boldsymbol{\Psi}_0 \boldsymbol{x}$ with $k\in\mathbb{N}$ filter applications in each).
\end{lem}

Beyond the information encoded in node features, graph wavelets encode geometric information even when it is not carried by input signals. Such a property has already been established, e.g., in the context of community detection, where white noise signals can be used in conjunction with graph wavelets to cluster nodes and reveal faithful community structures~\citep{roddenberry2020exact}.
To demonstrate a similar property in the context of GCN and scattering channels, we give an example of a simple graph structure with two cyclic substructures of different sizes (or cycle lengths) that are connected by one bottleneck edge. In this case, it can be verified that even with constant input signals, some geometric information is encoded by its convolution with graph filters as illustrated in Fig.~\ref{fig: toy example} (we refer the reader to the supplement for exact calculation of filter responses). However, as demonstrated in this case, while the information provided by the GCN filter responses $\boldsymbol{g}_\theta \star \boldsymbol{x}$ from Eq.~\ref{eq_conv. parametrization} is not constant, it does not distinguish between the two cyclic structures (and a similar pattern can be verified for $\boldsymbol{Ax}$). Formally, each node $u$ in one cycle is shown to have at least one node $v$ in the other with the same filter response (i.e., $\boldsymbol{g}_\theta \star \boldsymbol{x} (u) = \boldsymbol{g}_\theta \star \boldsymbol{x}(v)$). In contrast, the information extracted by the wavelet filter response $\boldsymbol{\Psi}_3 \boldsymbol{x}$ (used in geometric scattering) distinguishes between cycles and would allow for their separation. We note that this property generalizes to other cycle lengths as discussed in the supplement, but leave more extensive study of geometric information encoding in graph wavelets to future work.

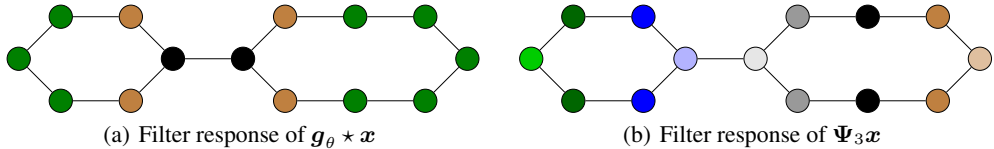
\begin{figure}[hb]
\centering
\subfigure[Filter response of $\boldsymbol{g}_\theta \star \boldsymbol{x}$]{\label{fig_toy example GCN}
\adjustbox{width=0.45\textwidth}{\begin{tikzpicture}
    \node[shape=circle,draw=black, fill=black!50!green] (0) at (0,0) {};
    \node[shape=circle,draw=black, fill=black!50!green] (1) at (0.6,0.6) {};
    \node[shape=circle,draw=black, fill=brown] (2) at (1.6,0.6) {};
    \node[shape=circle,draw=black, fill=black] (3) at (2.2,0) {};
    \node[shape=circle,draw=black, fill=brown] (4) at (1.6,-0.6) {};
    \node[shape=circle,draw=black, fill=black!50!green] (5) at (0.6,-0.6) {};
    
    \node[shape=circle,draw=black, fill=black] (6) at (3.2,0) {};
    \node[shape=circle,draw=black, fill=brown] (7) at (3.8,0.6) {};
    \node[shape=circle,draw=black, fill=black!50!green] (8) at (4.8,0.6) {};
    \node[shape=circle,draw=black, fill=black!50!green] (9) at (5.8,0.6) {};
    \node[shape=circle,draw=black, fill=black!50!green] (10) at (6.4,0) {};
    \node[shape=circle,draw=black, fill=black!50!green] (11) at (5.8,-0.6) {};
    \node[shape=circle,draw=black, fill=black!50!green] (12) at (4.8,-0.6) {};
    \node[shape=circle,draw=black, fill=brown] (13) at (3.8,-0.6) {};

    \path [-] (0) edge (1);
    \path [-] (1) edge (2);
    \path [-] (2) edge (3);
    \path [-] (3) edge (4);
    \path [-] (4) edge (5);
    \path [-] (5) edge (0);
    
    \path [-] (3) edge (6);
    \path [-] (6) edge (7);
    \path [-] (7) edge (8);
    \path [-] (8) edge (9);
    \path [-] (9) edge (10);
    \path [-] (10) edge (11);
    \path [-] (11) edge (12);
    \path [-] (12) edge (13);
    \path [-] (13) edge (6);
\end{tikzpicture}}}
\quad
\subfigure[Filter response of $\boldsymbol{\Psi}_3 \boldsymbol{x}$]{\label{fig_toy example scat}
\adjustbox{width=0.45\textwidth}{\begin{tikzpicture}
    \node[shape=circle,draw=black, fill=black!20!green] (0) at (0,0) {};
    \node[shape=circle,draw=black, fill=black!60!green] (1) at (0.6,0.6) {};
    \node[shape=circle,draw=black, fill=blue] (2) at (1.6,0.6) {};
    \node[shape=circle,draw=black, fill=white!70!blue] (3) at (2.2,0) {};
    \node[shape=circle,draw=black, fill=blue] (4) at (1.6,-0.6) {};
    \node[shape=circle,draw=black, fill=black!60!green] (5) at (0.6,-0.6) {};
    
    \node[shape=circle,draw=black, fill=white!90!black] (6) at (3.2,0) {};
    \node[shape=circle,draw=black, fill=white!60!black] (7) at (3.8,0.6) {};
    \node[shape=circle,draw=black, fill=black] (8) at (4.8,0.6) {};
    \node[shape=circle,draw=black, fill=brown] (9) at (5.8,0.6) {};
    \node[shape=circle,draw=black, fill=white!50!brown] (10) at (6.4,0) {};
    \node[shape=circle,draw=black, fill=brown] (11) at (5.8,-0.6) {};
    \node[shape=circle,draw=black, fill=black] (12) at (4.8,-0.6) {};
    \node[shape=circle,draw=black, fill=white!60!black] (13) at (3.8,-0.6) {};

    \path [-] (0) edge (1);
    \path [-] (1) edge (2);
    \path [-] (2) edge (3);
    \path [-] (3) edge (4);
    \path [-] (4) edge (5);
    \path [-] (5) edge (0);
    
    \path [-] (3) edge (6);
    \path [-] (6) edge (7);
    \path [-] (7) edge (8);
    \path [-] (8) edge (9);
    \path [-] (9) edge (10);
    \path [-] (10) edge (11);
    \path [-] (11) edge (12);
    \path [-] (12) edge (13);
    \path [-] (13) edge (6);
\end{tikzpicture}}}
\caption{Filter responses for \subref{fig_toy example GCN} the GCN filter (Eq.~\ref{eq_conv. parametrization}) and \subref{fig_toy example scat} a scattering filter applied to a constant signal $\boldsymbol{x}$ over a graph with two cyclic substuctures connected by a single-edge bottleneck. Color coding differs slightly between plots, but is consistent within each plot, indicating nodes with numerically indistinguishable response values.
}
\label{fig: toy example}
\end{figure}

\section{Empirical Results}\label{sect_results}
To evaluate our Scattering GCN approach, we compare it to several established methods for semi-supervised node classification, including the original GCN \citep{kipf2016semi}, which is known to be subject to the oversmoothing problem, as discussed in~\citep{li2018deeper}, and Sec.~\ref{sect_introduction} and~\ref{sect_graph conv. net} here. Further, we compare our approach with two recent methods that address the oversmoothing problem. The approach in~\cite{li2018deeper} directly addresses oversmoothing in GCNs by using partially absorbing random walks \citep{wu2012learning} to mitigate rapid mixing of node features in highly connected graph regions.
The graph attention network (GAT)~\citep{velivckovic2017graph} indirectly addresses oversmoothing by training adaptive node-wise weighting of the smoothing operation via an attention mechanism. Furthermore, we also include two alternatives to GCN networks based on Chebyshev polynomial filters~\citep{defferrard2016convolutional} and belief propagation of label information~\citep{zhu2003semi} computed via Gaussian random fields. Finally, we include two baseline approaches to verify the contribution of our hybrid approach compared to compared to the classifier from~\cite{zou2019graph} that is solely based on handcrafted graph-scattering features, and compared to SVM classifier acting directly on node features without considering graph edges, which does not incorporate any geometric information.

The methods from \citep{kipf2016semi,li2018deeper,velivckovic2017graph,defferrard2016convolutional,zhu2003semi} were all executed using the original implementations accompanying their publications. These are tuned and evaluated using the standard splits provided for the benchmark datasets for fair comparison. We ensure that the reported classification accuracies agree with previously published results when available. The tuning of our method (including hyperparameters and composition of GCN and scattering channels) on each dataset was done via grid search (over a fixed set of choices for all datasets) using the same cross validation setup used to tune competing methods. For further details, we refer the reader to the supplement, which contains an ablation study evaluating the importance of each component in our proposed architecture.

\begin{wraptable}[10]{r}{0.65\textwidth}
\vspace{-10pt}
\caption{Dataset characteristics: number of nodes, edges, and features; mean $\pm$ std.\ of node degrees; ratio of \#edges to \#nodes.}
\label{table:dataset}
\centering
\begin{tabular}{cccccccc}
\toprule
Dataset & Nodes & Edges& Features & Degrees  & $\frac{\text{Edges}}{\text{Nodes}}$\\
\midrule
Citeseer & 3,327 & 4,732 & 3,703 & 3.77$\pm$3.38 & 1.42 \\
Cora & 2,708 & 5,429 & 1,433 & 4.90$\pm$5.22 & 2.00 \\
Pubmed & 19,717 & 44,338 & 500 & 5.50$\pm$7.43 & 2.25 \\
DBLP & 17,716 & 52,867 & 1639 & 6.97$\pm$9.35 & 2.98 \\
\bottomrule
\end{tabular}
\end{wraptable}

Our comparisons are based on four popular graph datasets with varying sizes and connectivity structures summarized in Tab.~\ref{table:dataset} (see, e.g.,~\cite{yang2016revisiting} for Citeseer, Cora, and Pubmed, and~\cite{pan2016tri} for DBLP).
We order the datasets by increasing connectivity structure, reflected by their node degrees and edges-to-nodes ratios. As discussed in~\citep{li2018deeper}, increased connectivity leads to faster mixing of node features in GCN, exacerbating the oversmoothing problem (as nodes quickly become indistinguishable) and degrading classification performance. Therefore, we expect the impact of scattering channels and the relative improvement achieved by Scattering GCN to correspond to the increasing connectivity order of datasets in Tab.~\ref{table:dataset}, which is maintained for our reported results in Tab.~\ref{tab_test accuracies} and Fig.~\ref{fig:results}. 

\begin{wraptable}[17]{r}{0.68\textwidth}
\vspace{-12pt}
\caption{Classification accuracy (top two marked in bold; best one underlined) of Scattering GCN on four benchmark datasets compared to four other GNNs~\cite{velivckovic2017graph,li2018deeper,kipf2016semi,defferrard2016convolutional}, a non-GNN approach \cite{zhu2003semi} based on belief propagation, a pure graph scattering baseline~\cite{zou2019graph}, and a nongeometric baseline only using node features with linear SVM.}
\label{tab_test accuracies}
\centering
\begin{tabular}{lcccr}
\toprule
Model & Citeseer & Cora & Pubmed & DBLP \\
\midrule
Scattering GCN (ours) & \textbf{71.7} & \underline{\textbf{84.2}} & \underline{\textbf{79.4}}& \underline{\textbf{81.5}} \\
GAT~\cite{velivckovic2017graph} & \underline{\textbf{72.5}} & \textbf{83.0} & 79.0 & 66.1 \\
Partially absorbing~\cite{li2018deeper} & 71.2 & 81.7 & \textbf{79.2} & 56.9 \\
GCN~\cite{kipf2016semi} & 70.3 & 81.5 & 79.0 & 59.3\\
Chebyshev~\cite{defferrard2016convolutional} & 69.8 & 78.1 & 74.4 & 57.3 \\
Label Propagation~\cite{zhu2003semi} & 58.2 & 77.3 & 71.0 & 53.0 \\
\midrule
Graph scattering~\cite{zou2019graph} &  67.5 & 81.9 & 69.8 & \textbf{69.4} \\
\midrule
Node features (SVM) &  61.1 & 58.0 & 49.9 & 48.2 \\

\bottomrule
\end{tabular}
\end{wraptable}

We first consider test classification accuracy reported in Tab.~\ref{tab_test accuracies}, which shows that our approach outperforms other methods on three out of the four considered datasets. On the remaining one (namely Citeseer) we are only outperformed by GAT. However, we note that this dataset has the weakest connectivity structure (see Tab.~\ref{table:dataset}) and the most informative node features (e.g., achieving 61.1\% accuracy via linear SVM without considering any graph information). In contrast, on DBLP, which has the richest connectivity structure and least informative features (only 48.2\% SVM accuracy), we significantly outperform GAT (over 15\% improvement), which itself significantly outperforms all other methods (by 6.8\% or more) except for the graph scattering baseline from~\cite{zou2019graph}.

\begin{figure}
\centering
\subfigure[Citeseer]{\label{fig:cite_size}
\begin{minipage}{0.232\textwidth}
\includegraphics[width=\textwidth,trim={12pt 12pt 12pt 12pt},clip]{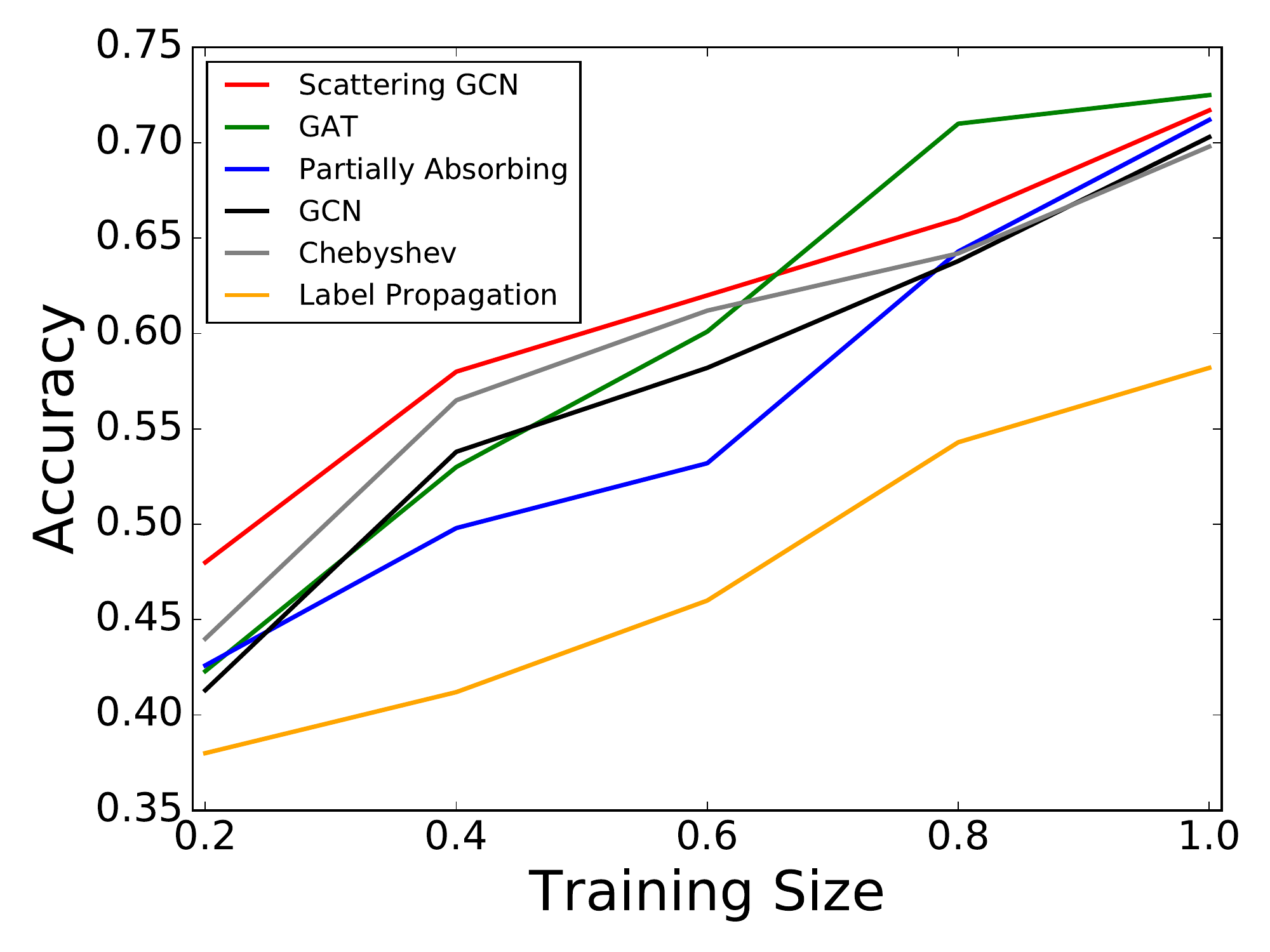}\\[5pt]
\includegraphics[width=\textwidth,trim={10pt 0pt 12pt 10pt},clip]{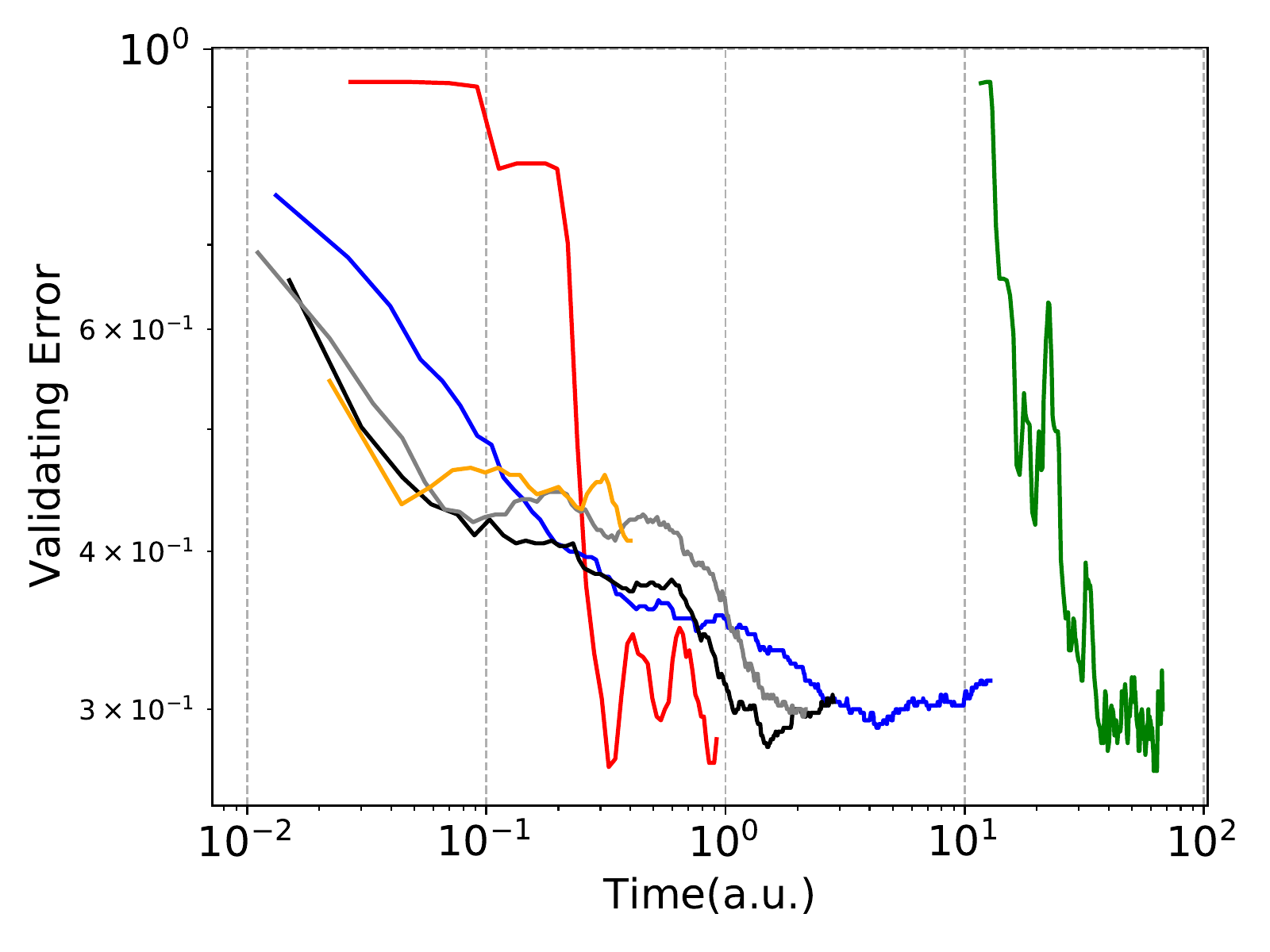}
\end{minipage}
}
\subfigure[Cora]{\label{fig:cora_size}
\begin{minipage}{0.232\textwidth}
\includegraphics[width=\textwidth,trim={12pt 12pt 12pt 12pt},clip]{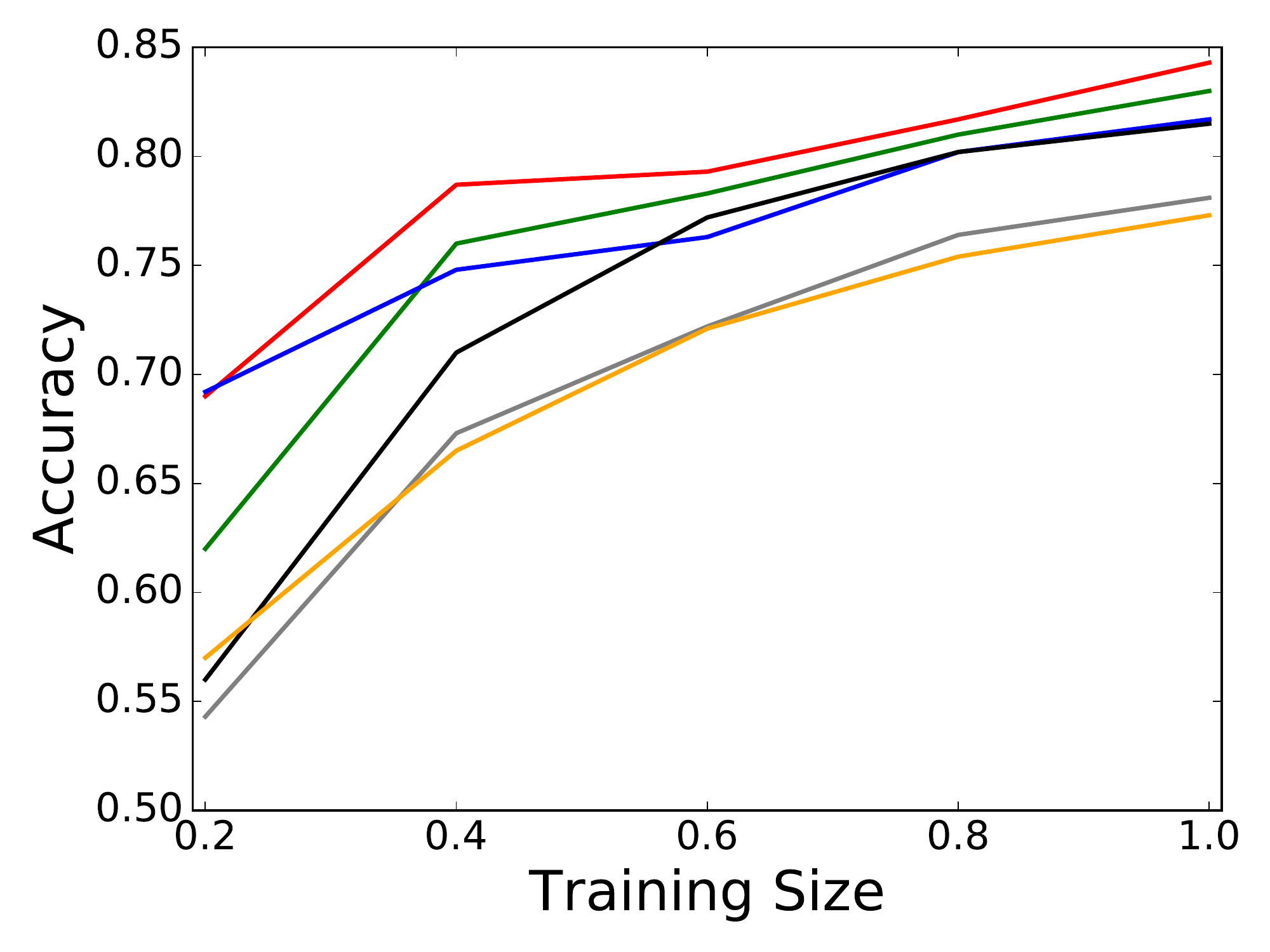}\\[5pt]
\includegraphics[width=\textwidth,trim={10pt 0pt 12pt 10pt},clip]{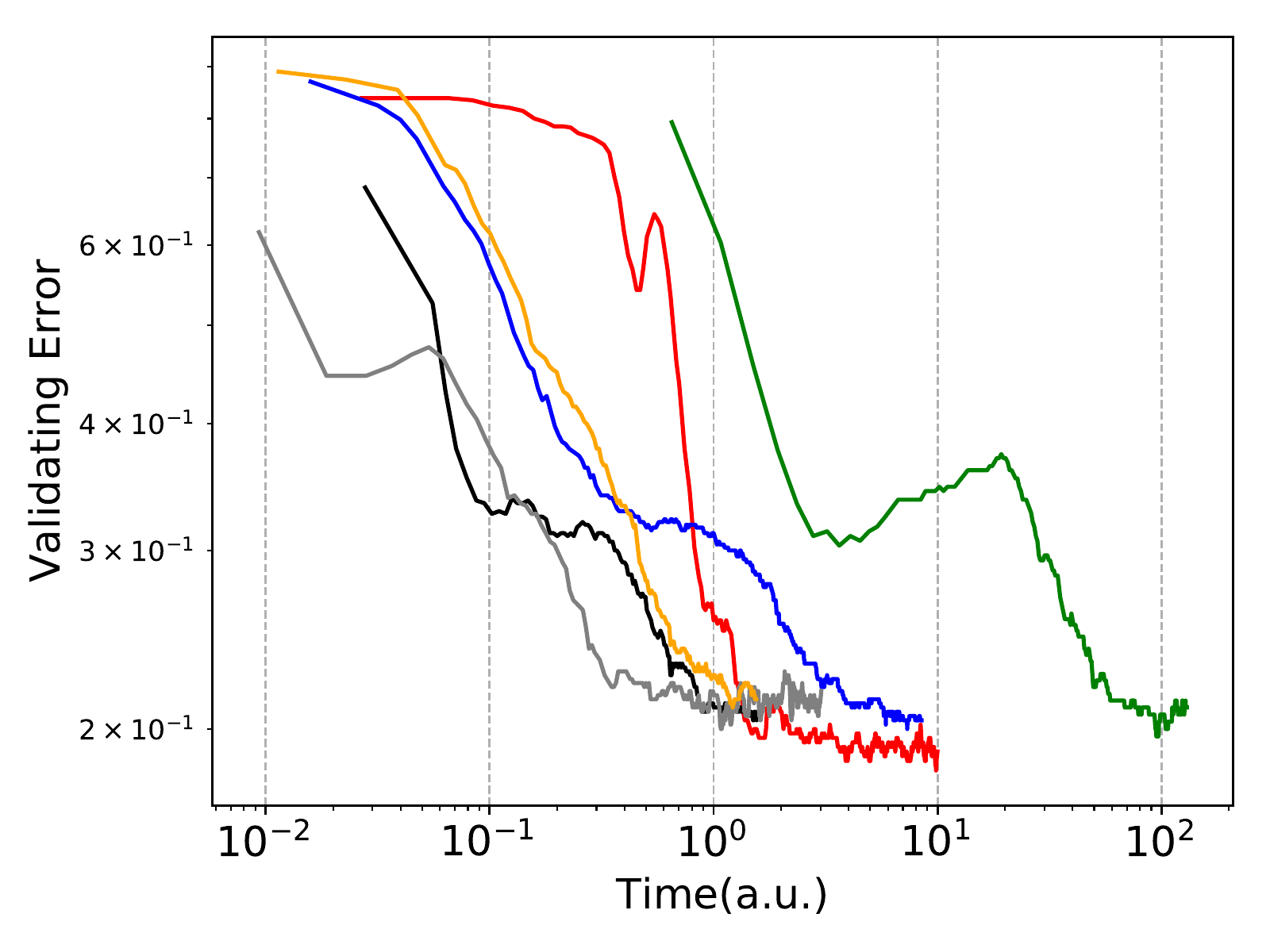}
\end{minipage}
}
\subfigure[Pubmed]{\label{fig:pubmed_size}
\begin{minipage}{0.232\textwidth}
\includegraphics[width=\textwidth,trim={12pt 12pt 12pt 12pt},clip]{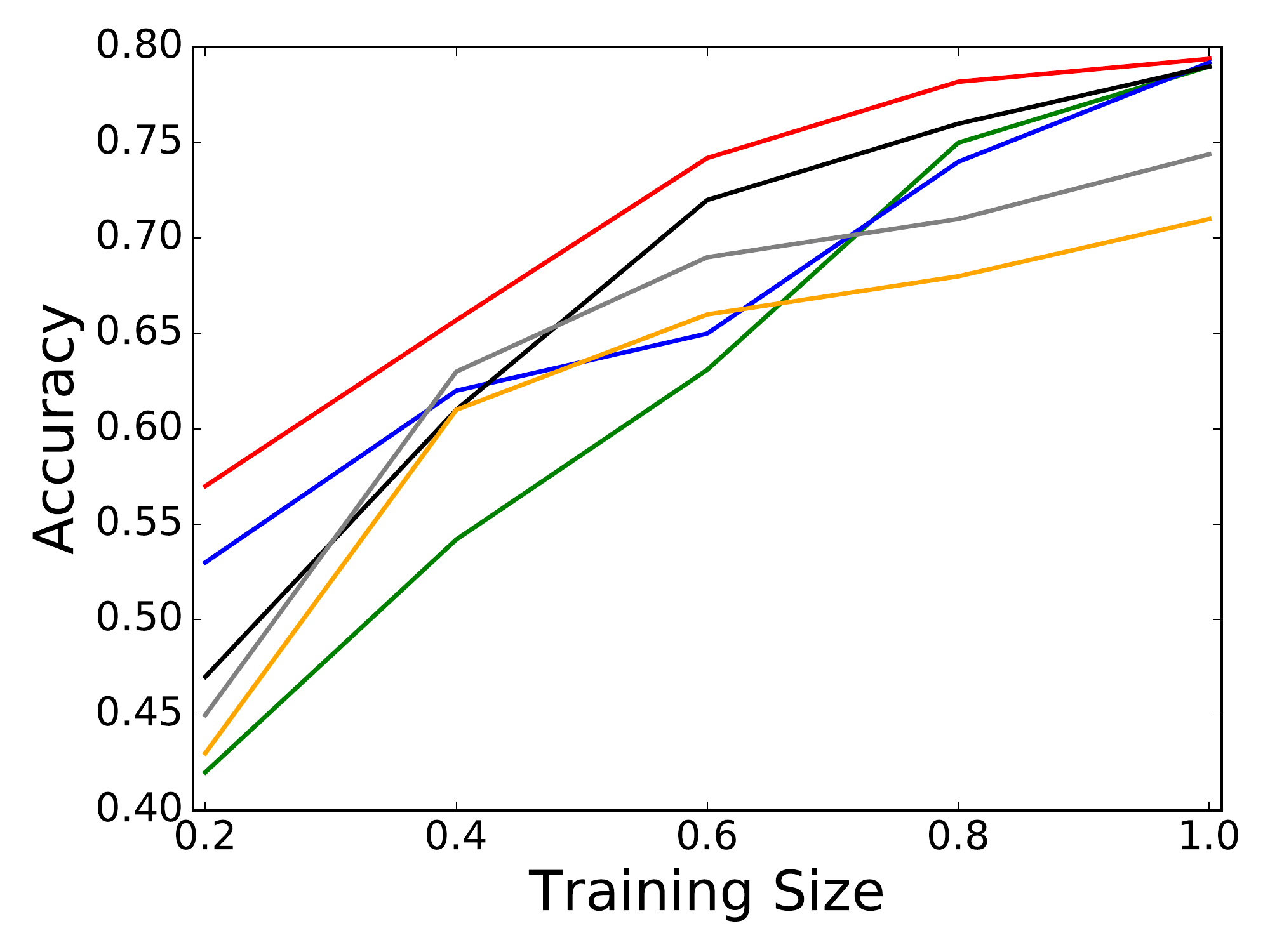}\\[5pt]
\includegraphics[width=\textwidth,trim={10pt 0pt 12pt 10pt},clip]{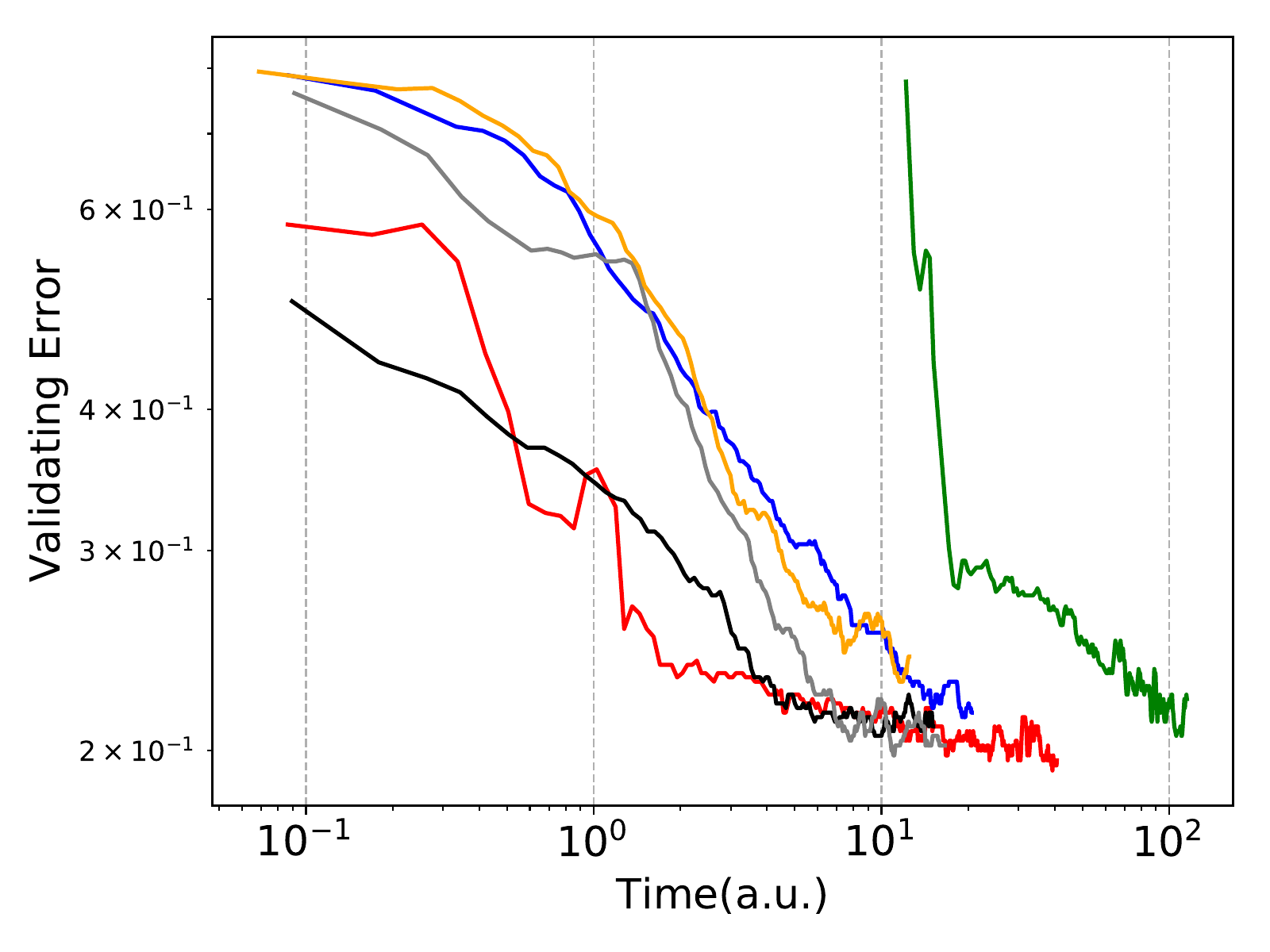}
\end{minipage}
}
\subfigure[DBLP]{\label{fig:dblp_size}
\begin{minipage}{0.232\textwidth}
\includegraphics[width=\textwidth,trim={12pt 12pt 12pt 12pt},clip]{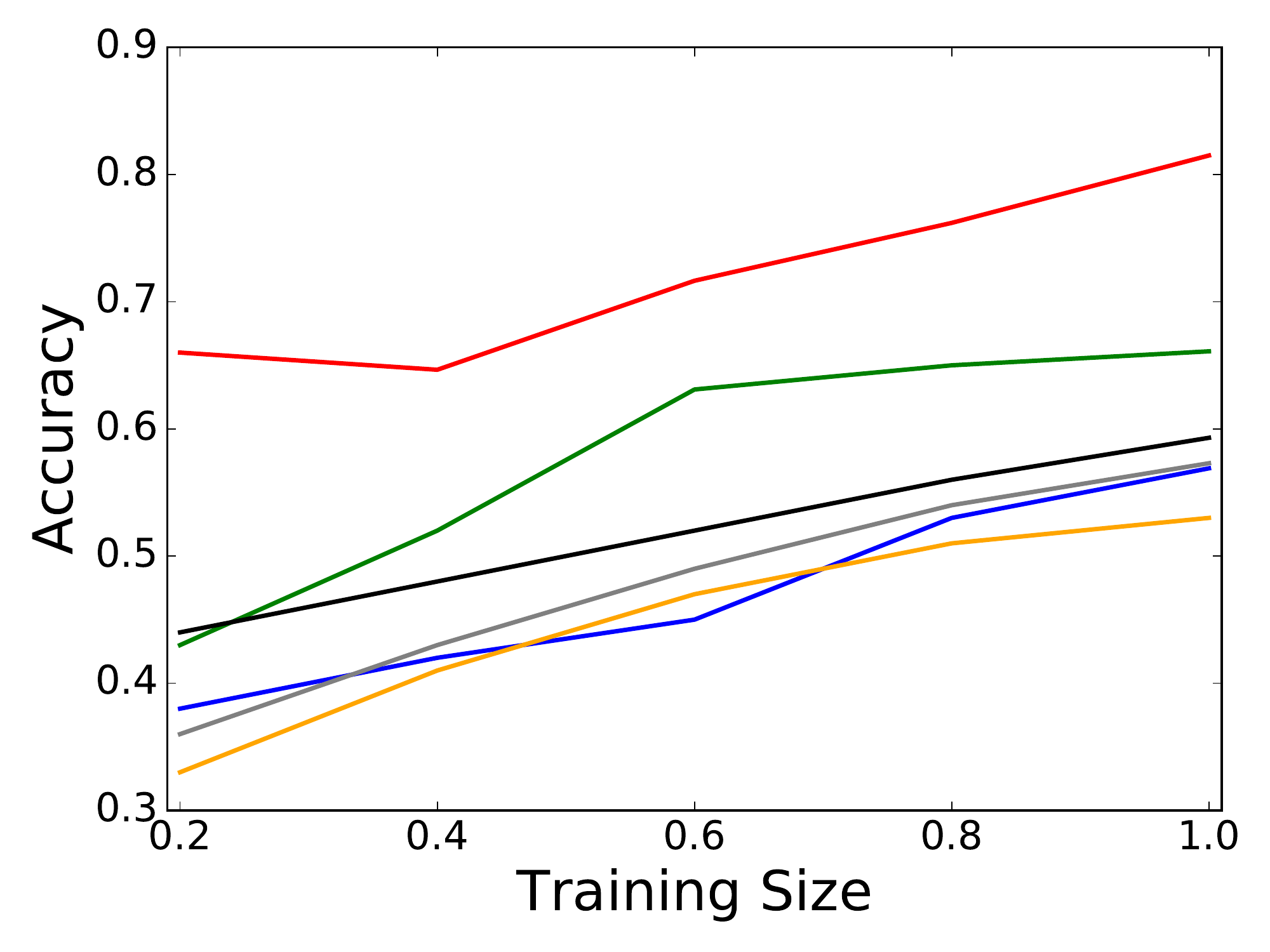}\\[5pt]
\includegraphics[width=\textwidth,trim={10pt 0pt 12pt 10pt},clip]{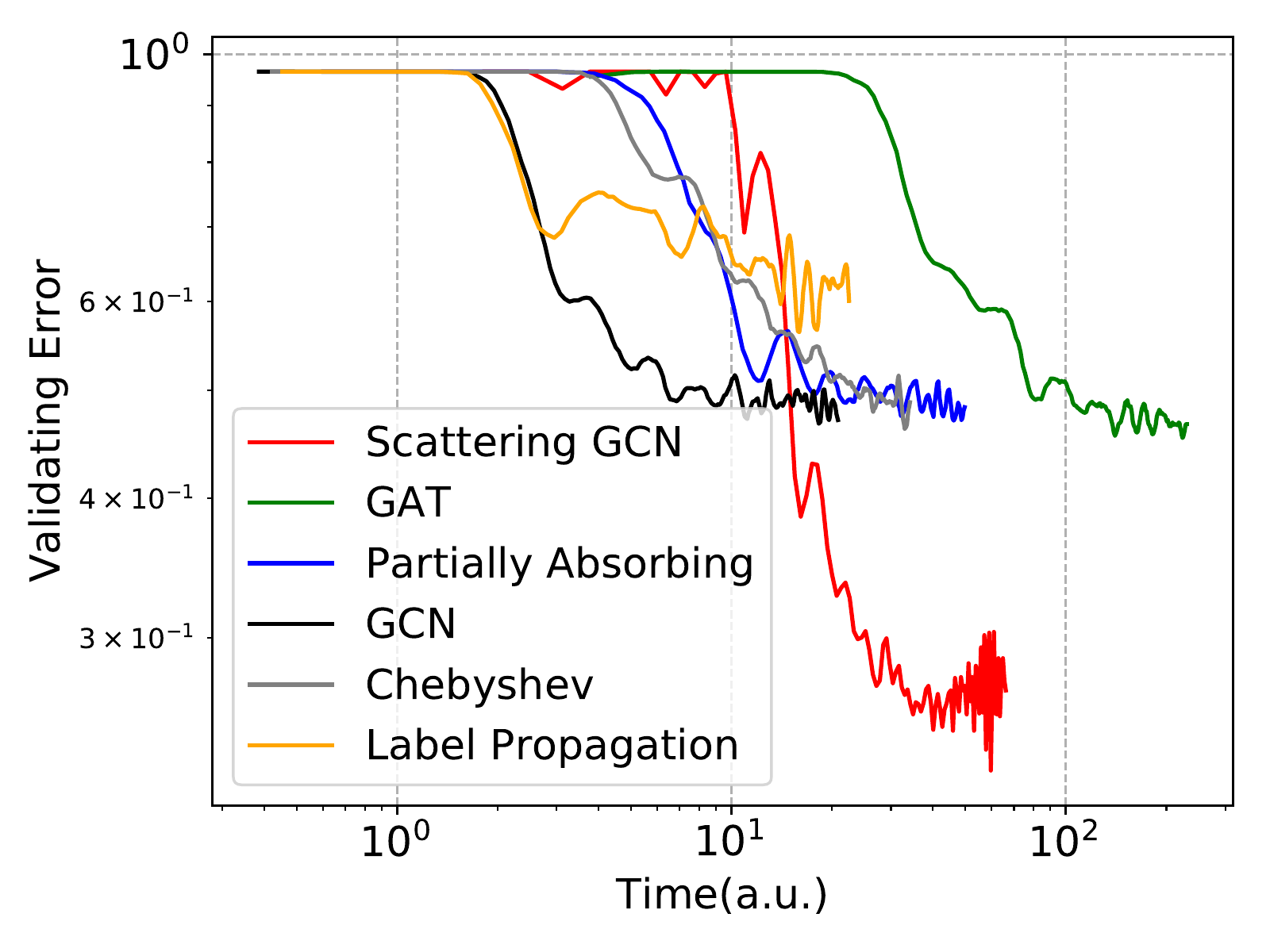}
\end{minipage}
}
\caption{Impact of training set size (top) and training time (bottom) on classification accuracy and error (correspondingly); training size measured relative to the original training size of each dataset; training time and validation error plotted in logarithmic scale; runtime measured for all methods on the same hardware, using original implementations accompanying their publications.}
\label{fig:results}
\end{figure}

Next, we consider the impact of training size on classification performance as we are interested in semi-supervised settings where only a small portion of nodes in the graph are labelled. Fig.~\ref{fig:results} (top) presents the classification accuracy (on validation set) for the training size reduced to 20\%, 40\%, 60\%, 80\% and 100\% of the original training size available for each dataset. These results indicate that Scattering GCN generally exhibits greater stability to sparse training conditions compared to other methods. Importantly, we note that on Citeseer, while GAT outperforms our method for the original training size, its performance degrades rapidly when training size is reduced below 60\% of the original one, at which point Scattering GCN outperforms all other methods. We also note that on Pubmed, even a small decrease in training size (e.g., 80\% of original) creates a significant performance gap between Scattering GCN and GAT, which we believe is due to node features being less independently informative in this case (see baseline in Tab.~\ref{tab_test accuracies}) compared to Citeseer and Cora. 

Finally, in Fig.~\ref{fig:results} (bottom), we consider the evolution of (validation) classification error during the training process. Overall, our results indicate that the training of Scattering GCN reaches low validation errors significantly faster than Partially Absorbing and GAT\footnote{The horizontal shift shown for GAT in Fig.~\ref{fig:results} (bottom), indicating increased training runtime (based on the original implementation accompanying~\cite{velivckovic2017graph}), could be explained by its optimization process requiring more weights than other methods and an intensive gradient computations driven not only graph nodes, but also by graph edges considered in the multihead attention mechanism.}, which are the two other leading methods (in terms of final test accuracy in Tab.~\ref{tab_test accuracies}). On Pubmed, which is the largest dataset considered here (by number of nodes), our error decays at a similar rate to that of GCN, showing a notable gap over all other methods. On DBLP, which has a similar number of nodes but significantly more edges, Scattering GCN takes longer to converge (compared to GCN), but as discussed before, it also demonstrates a significant (double-digit) performance lead compared to all other methods.

\section{Conclusion}\label{sect_conclusion}
Our study of semi-supervised node-level classification tasks for graphs presents a new approach to address some of the main concerns and limitations of GCN models. We discuss and consider richer notions of regularity on graphs to expand the GCN approach, which solely relies on enforcing smoothness over graph neighborhoods. This is achieved by incorporating multiple frequency bands of graph signals, which are typically not leveraged in traditional GCN models. Our construction is inspired by geometric scattering, which has mainly been used for whole-graph classification so far. Our results demonstrate several benefits of incorporating the elements presented here (i.e., scattering channels and residual convolutions) into GCN architectures. Furthermore, we expect the incorporation of these elements together in more intricate architectures to provide new capabilities of pattern recognition and local information extraction in graphs. For example, attention mechanisms could be used to adaptively tune scattering configurations at the resolution of each node, rather than the global graph level used here. We leave the exploration of such research avenues for future work.

\section*{Broader Impact}

Node classification in graphs is an important task that gains increasing interest nowadays in multiple fields looking into network analysis applications. For example, they are of interest in social studies, where a natural application is the study of social networks and other interaction graphs. Other popular application fields include biochemistry and epidemiology. However, this work is computational in nature and addresses the foundations of graph processing and geometric deep learning. As such, by itself, it is not expected to raise ethical concerns nor to have adverse effects on society.

\begin{ack}
The authors would like to thank Dongmian Zou for fruitful discussions. This work was partially funded by IVADO Professor startup \& operational funds, IVADO Fundamental Research Project grant PRF-2019-3583139727, and NIH grant R01GM135929. The content provided here is solely the responsibility of the authors and does not necessarily represent the official views of the funding agencies.
\end{ack}

\bibliographystyle{unsrtnat}
\bibliography{references}

\clearpage

\section*{Supplement}
\appendix
\section{Proofs and Illustrative Examples of Lemmas~1 and~2}

\begin{figure}[!h]
    \centering
    \subfigure[Bipartite graph]{\label{subfig_bipartite}
    \adjustbox{width=0.48\textwidth}{
    \includegraphics[width=\linewidth]{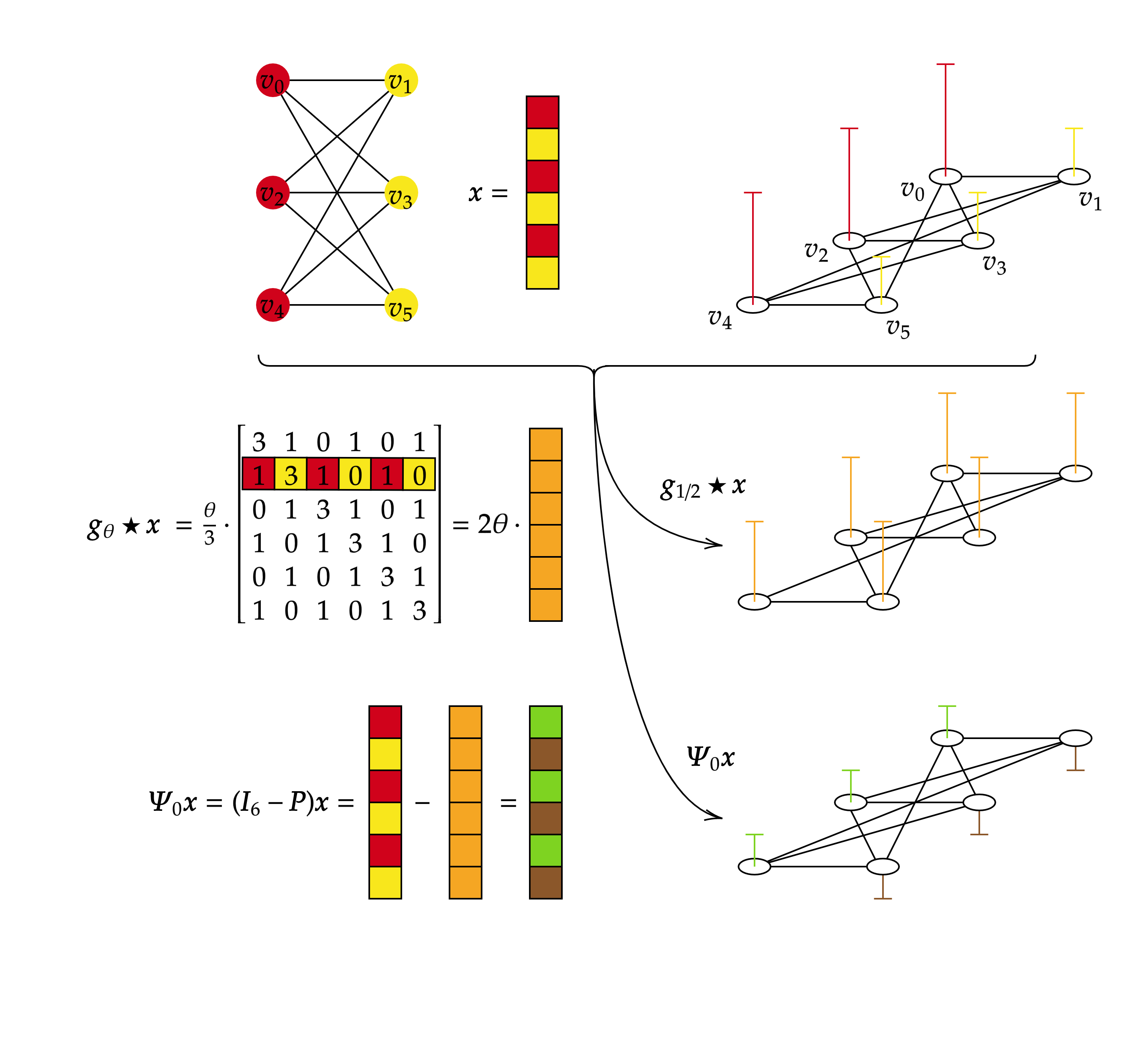}
    }}
    \subfigure[Cyclic graph]{\label{subfig_cyclic}
    \adjustbox{width=0.48\textwidth}{
    \includegraphics[width=\linewidth]{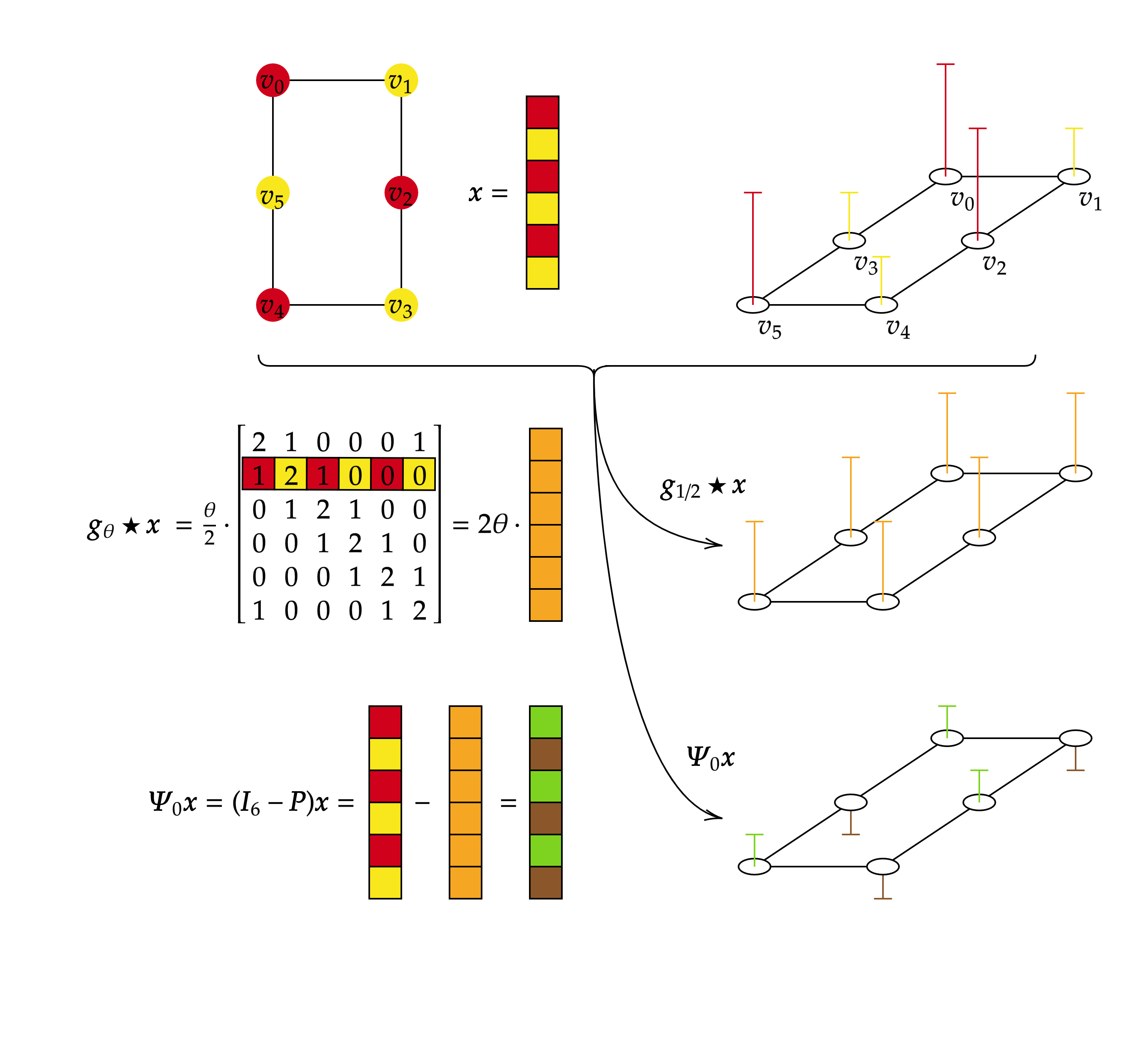}
    }}
    
    \label{fig_theory}
    \caption{Illustrative examples for Lemma 1 and 2 in Sec. 7 of the main paper.}
\end{figure}

\begin{proof}[Proof of Lemma 1]
Note that $G_{cyc}^{2n}$ is a bipartite graph with constant node degree $\beta=2$. Therefore, the proof of Lemma 1 can be seen as special case of Lemma 2, which is proved below.
\end{proof}

\begin{proof}[Proof of Lemma 2]
We first notice that if $\boldsymbol{g}_{\frac{1}{2}} \star \boldsymbol{x}$ is constant, then $\boldsymbol{g}_\theta \star \boldsymbol{x} = 2 \theta (\boldsymbol{g}_{\frac{1}{2}} \star \boldsymbol{x})$ is constant for any $\theta$. Furthermore, for the considered class of graphs, $\boldsymbol{D} = \beta \Id_n$ with $\beta>0$, implying that $\boldsymbol{D}^{-1}=\frac{1}{\beta} \Id_n$ and $\boldsymbol{D}^{-1/2} = \frac{1}{\sqrt{\beta}} \Id_n$. Therefore, as a direct result of Eq.~1 in the main paper, it holds that
\begin{equation}\label{eq_A}
    \boldsymbol{g}_{\frac{1}{2}} \star \boldsymbol{x} = \left(\frac{1}{2} \Id_n + \frac{1}{2\beta}\boldsymbol{W}\right)\boldsymbol{x} = \boldsymbol{P}\boldsymbol{x}.
\end{equation}
Similarly, it is easily verified that any $k \in \mathbb{N}$ applications of the convolution with $\boldsymbol{g}_\theta$ (for any $\theta \in \mathbb{R}$) can be written as $2^k\theta^k \boldsymbol{P}^k \boldsymbol{x}$.
Furthermore, since $\boldsymbol{P}$ is column-stochastic and (here) symmetric (thus also row-stochastic), we have $\boldsymbol{P} \boldsymbol{c} = \boldsymbol{c}$ for any constant signal $\boldsymbol{c} = c \boldsymbol{1}_{2n}$. Thus, it is sufficient to show that $\boldsymbol{P}\boldsymbol{x}$ is a constant signal to verify the first claim of the lemma.

We consider $G_{bi,\beta}^n = (V,E)$. For any node $v\in V$, according to Eq. \ref{eq_A}, we can write
\begin{align*}
    (\boldsymbol{P} \boldsymbol{x})[v]
    = \overbrace{\boldsymbol{P}[v,v]}^{=\frac{1}{2}} \boldsymbol{x}[v] + \sum_{u\in \mathcal{N}(v)} \overbrace{\boldsymbol{P}[v,u]}^{=\frac{1}{2\beta}} \boldsymbol{x}[u] + \sum_{w\in V^{v}} \overbrace{\boldsymbol{P}[v,w]}^{=0} \boldsymbol{x}[w], %\notag \\
    %= \frac{\boldsymbol{x}[v]}{2} + \sum_{u\in\mathcal{N}(v)} \frac{\boldsymbol{x}[u]}{2\beta},
\end{align*}
where we denote by $\mathcal{N}(v)$ the neighborhood of the node $v$ and set $V^{v}\coloneqq V\setminus (\{v\}\cup \mathcal{N}(v))$. This implies that
\begin{equation}\label{eq_Ax[v]}
    (\boldsymbol{P} \boldsymbol{x})[v] = \frac{\boldsymbol{x}[v]}{2} + \sum_{u\in\mathcal{N}(v)} \frac{\boldsymbol{x}[u]}{2\beta}.
\end{equation}

We now consider a 2-coloring signal $\boldsymbol{x}\in\mathbb{R}^n$. W.l.o.g., let $\boldsymbol{x}[v] = a$, which implies $\boldsymbol{x}[u] = b$ for all $u\in\mathcal{N}(v)$.
Now, since $\vert\mathcal{N}(v)\vert=\beta$, it holds
$$
    (\boldsymbol{P} \boldsymbol{x})[v] = \frac{a+b}{2},
$$
thus verifying the first claim of the lemma as the choice of $v$ was arbitrary.
% similarly
% $$
%     (\boldsymbol{g}_\theta \star \boldsymbol{x})[v] = \theta (a+b).
% $$
% We deduce that the vectors $\boldsymbol{Px}$ and $\boldsymbol{g}_\theta \star \boldsymbol{x}$ are constant as the choice of $v$ was arbitrary. 
Finally, it is now straightforward to verify the second claim as well, since the operation $\boldsymbol{\Psi}_0 \boldsymbol{x} = (\Id_n - \boldsymbol{P) \boldsymbol{x}} = \boldsymbol{x} - \frac{a+b}{2} \ones_n$ retains a 2-coloring signal (the original colors are shifted by a constant: $-\frac{a+b}{2}$).
\end{proof}

% \begin{figure}
%     \centering
%     \includegraphics[width=\linewidth]{Figures/visual_cyclic.png}
%     \caption{Caption}
%     \label{fig_bipartite}
% \end{figure}

\section{Geometric Information Encoded by Graph Wavelets (supp.\ info.\ Sec.~7)}
Let $\mathcal{G}_{cyc}^{\geq\ell}\coloneqq \{G_{cyc}^k : k\geq \ell\}$, $3\leq\ell\in\mathbb{N}$, be the class of unweighted cyclic graphs of length greater than or equal to $\ell$. 
%Furthermore, let $\mathcal{G}_\star^{\geq\ell}$ be the class of graphs constructed by taking $n \geq 2$ cyclic graphs $G_1,\dots,G_n\in \mathcal{G}_{cyc}^{\geq \ell}$, arranging them in the order of the indexes, and connecting subsequent cycles with bottleneck edges such that
Furthermore, let $\mathcal{G}_{\ell}^\star$ be the class of graphs constructed by taking $n \geq 2$ cyclic graphs $G_1,\dots,G_n$, arranging them in the order of the indexes, and connecting subsequent cycles with bottleneck edges as described in the following.
\begin{enumerate}
    \item We take $G_1, G_n\in \mathcal{G}_{cyc}^{\geq 4}$ and $G_2,\dots, G_{n-1}\in\mathcal{G}_{cyc}^{\geq\ell}$.
    \item The (sub)graph $G_i$, $1 \leq i \leq n - 1$, is connected by exactly one bottleneck edge to $G_{i+1}$.
    \item No edge is connecting (sub)graphs $G_i$ and $G_j$ if $|i - j| \geq 2$.
    \item For each $G_i$, $2 \leq i \leq n - 1$, there are exactly two nodes in $G_i$ with bottleneck edges coming out of them, and these nodes are the farthest from each other in the cycle \footnote{For cycles of odd length, the choice of the node to connect is ambiguous (as there are two qualifying nodes), but the claim holds for either choice} (in shortest-path distance).
\end{enumerate}
This construction essentially generalizes the graph demonstrated in Fig.~3 of the main paper (see Sec.~7). The following lemma shows that on such graphs, the filter responses of $\boldsymbol{g}_\theta$ for a constant signal will encode some geometric information, but will not distinguish between the cycles in the graph. Note that this result can also be generalized further to a chain that is closed by connecting $G_n$ and $G_1$ with a bottleneck edge if we further assume that $G_1, G_n\in\mathcal{G}_{cyc}^{\geq 7}$.

\begin{lem} \label{lem:supp}
Let $G=(V, E)$ a graph of the class $\mathcal{G}_{7}^\star$. We consider a constant signal $\boldsymbol{c} = c \boldsymbol{1}_{\vert V\vert}$, for some $c \in \mathbb{R}$. Then, for all nodes $v\in V$ and for any $\theta \in \mathbb{R}$, the filter response $(\boldsymbol{g}_\theta \star \boldsymbol{c})[v]$ shares its value with at least one node of each other cyclic substructure.
\end{lem}

\begin{proof}
% To prove this claim, we need to study the structure of $\boldsymbol{A}$ as given in Sec.~3 of the main paper. 
First, note that $V$ contains only the following two kinds of nodes. We refer to a node of degree 3 (those contained in a bottleneck edge) as a \textit{hub}, while using the term \textit{pass} for all other nodes (those of degree 2). Furthermore, due to the minimal cycle length, and the shortest-path distance requirement between hub nodes in the same cycle, only three types of neighborhoods can be encountered. Indeed, it is easy to see that each hub node has one hub neighbor and two pass neighbors, while pass nodes can either have two pass neighbors or one pass and one hub.

% It turns out, that given the assumed graph structure, $\boldsymbol{A}$ only contains 4 different values and is given by
% $$
%     \boldsymbol{A}[v,w]=
%     \begin{cases}
%         \frac{1}{4}, \quad \text{ if $v=w$ is a hub}, \\
%         \frac{1}{3}, \quad \text{ if $v=w$ is a pass}, \\
%         \frac{1}{4}, \quad \text{ if $v\neq w$ are both hubs}, \\
%         \frac{1}{3}, \quad \text{ if $v\neq w$ are both passes}, \\
%         \frac{1}{2\sqrt{3}}, \quad \text{ if $v$ is a hub and $w$ is a pass}, \\
        
%         0, \quad \text{ otherwise}.
%     \end{cases}
% $$
% $$
%     \boldsymbol{A}[v,w]=
%     \begin{cases}
%         \frac{1}{4} & \text{ if $v,w$ are hubs}, \\
%         \frac{1}{3} & \text{ if $v,w$ are passes}, \\
%         \frac{1}{2\sqrt{3}} & \text{ if $v$ is a hub and $w$ is a pass}, \\
        
%         0 & \text{ otherwise}.
%     \end{cases}
% $$

% Now, as the input signal is constant ($\boldsymbol{c}=c\ones_{\vert V\vert}$ for some constant $c\in\mathbb{R}$), we can identify the node feature of each node with the sum of the values in the corresponding row in $\boldsymbol{A}$ (up to scaling with $c$). This in turn implies, that the ordering of the values per row does not matter and we can process them as multisets instead.

Next, we notice that for any $\theta$ and any $c$, the filter response $(\boldsymbol{g}_{\theta} \star \boldsymbol{c})[v] = c \theta (\boldsymbol{g}_1 \star \boldsymbol{1}_{\vert V \vert})[v]$ is fully determined by the neighborhood type of $v \in V$. 
% To conclude the proof, we will now show that only three different kinds of nodes and thus associated rows (seen as multisets) exist and moreover, that there is at least one node in every cyclic substructure associated to such a row. 
Therefore, computing these boils down to a simple proof by cases:
\begin{enumerate}
    \item Let $v$ be a hub, then the corresponding response is $(\boldsymbol{g}_{\theta} \star \boldsymbol{c})[v] = c \, \theta \,  (1 + \frac{1}{3} + \frac{1}{\sqrt{6}} + \frac{1}{\sqrt{6}}) \approx 2.150 \cdot c \, \theta$.
    \item Let $v$ be a pass with two pass neighbors, then the corresponding response is $(\boldsymbol{g}_{\theta} \star \boldsymbol{c})[v]= c \, \theta \, (1 + \frac{1}{2} + \frac{1}{2}) = 2 \cdot c \, \theta$.
    \item Let $v$ be a pass with one hub neighbor and one pass neighbor, then the corresponding response is $(\boldsymbol{g}_{\theta} \star \boldsymbol{c})[v]=c \, \theta \, (1 + \frac{1}{\sqrt{6}} + \frac{1}{2}) \approx 1.908 \cdot c \, \theta$.
\end{enumerate} 
Finally, it is trivial to see that every cyclic substructure from the class $\mathcal{G}_7^\star$ contains at least one hub and two passes connected to that hub. The existence of a pass only connected to other passes follows from the choice of the minimal cycle lengths together with the requirement that two hubs within a cycle have maximal distance from each other.
\end{proof}

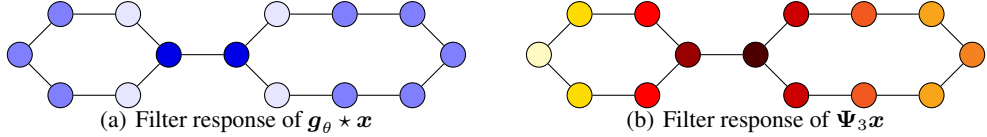
\begin{figure}[htb]
\centering

\subfigure[Filter response of $\boldsymbol{g}_\theta \star \boldsymbol{x}$]{
\label{fig_ex_GCN}
\begin{minipage}{0.45\columnwidth}
\centering
\begin{tikzpicture}[scale=0.9]
    \node[shape=circle,draw=black, fill=blue!50!white] (0) at (0,0) {};
    \node[shape=circle,draw=black, fill=blue!50!white] (1) at (0.6,0.6) {};
    \node[shape=circle,draw=black, fill=blue!10!white] (2) at (1.6,0.6) {};
    \node[shape=circle,draw=black, fill=black!10!blue] (3) at (2.2,0) {};
    \node[shape=circle,draw=black, fill=blue!10!white] (4) at (1.6,-0.6) {};
    \node[shape=circle,draw=black, fill=blue!50!white] (5) at (0.6,-0.6) {};
    
    \node[shape=circle,draw=black, fill=black!10!blue] (6) at (3.2,0) {};
    \node[shape=circle,draw=black, fill=blue!10!white] (7) at (3.8,0.6) {};
    \node[shape=circle,draw=black, fill=blue!50!white] (8) at (4.8,0.6) {};
    \node[shape=circle,draw=black, fill=blue!50!white] (9) at (5.8,0.6) {};
    \node[shape=circle,draw=black, fill=blue!50!white] (10) at (6.4,0) {};
    \node[shape=circle,draw=black, fill=blue!50!white] (11) at (5.8,-0.6) {};
    \node[shape=circle,draw=black, fill=blue!50!white] (12) at (4.8,-0.6) {};
    \node[shape=circle,draw=black, fill=blue!10!white] (13) at (3.8,-0.6) {};

    \path [-] (0) edge (1);
    \path [-] (1) edge (2);
    \path [-] (2) edge (3);
    \path [-] (3) edge (4);
    \path [-] (4) edge (5);
    \path [-] (5) edge (0);
    
    \path [-] (3) edge (6);
    \path [-] (6) edge (7);
    \path [-] (7) edge (8);
    \path [-] (8) edge (9);
    \path [-] (9) edge (10);
    \path [-] (10) edge (11);
    \path [-] (11) edge (12);
    \path [-] (12) edge (13);
    \path [-] (13) edge (6);
\end{tikzpicture}
\end{minipage}
}
\quad
\subfigure[Filter response of $\boldsymbol{\Psi}_3 \boldsymbol{x}$]{
\label{fig_ex_Scat}
\begin{minipage}{0.45\columnwidth}
\centering
\begin{tikzpicture}[scale=0.9]
    \node[shape=circle,draw=black, fill=white!70!yellow] (0) at (0,0) {};
    \node[shape=circle,draw=black, fill=yellow!90!red] (1) at (0.6,0.6) {};
    \node[shape=circle,draw=black, fill=red] (2) at (1.6,0.6) {};
    \node[shape=circle,draw=black, fill=black!40!red] (3) at (2.2,0) {};
    \node[shape=circle,draw=black, fill=red] (4) at (1.6,-0.6) {};
    \node[shape=circle,draw=black, fill=yellow!90!red] (5) at (0.6,-0.6) {};
    
    \node[shape=circle,draw=black, fill=black!70!red] (6) at (3.2,0) {};
    \node[shape=circle,draw=black, fill=black!20!red] (7) at (3.8,0.6) {};
    \node[shape=circle,draw=black, fill=yellow!20!red] (8) at (4.8,0.6) {};
    \node[shape=circle,draw=black, fill=yellow!60!red] (9) at (5.8,0.6) {};
    \node[shape=circle,draw=black, fill=yellow!40!red] (10) at (6.4,0) {};
    \node[shape=circle,draw=black, fill=yellow!60!red] (11) at (5.8,-0.6) {};
    \node[shape=circle,draw=black, fill=yellow!20!red] (12) at (4.8,-0.6) {};
    \node[shape=circle,draw=black, fill=black!20!red] (13) at (3.8,-0.6) {};

    \path [-] (0) edge (1);
    \path [-] (1) edge (2);
    \path [-] (2) edge (3);
    \path [-] (3) edge (4);
    \path [-] (4) edge (5);
    \path [-] (5) edge (0);
    
    \path [-] (3) edge (6);
    \path [-] (6) edge (7);
    \path [-] (7) edge (8);
    \path [-] (8) edge (9);
    \path [-] (9) edge (10);
    \path [-] (10) edge (11);
    \path [-] (11) edge (12);
    \path [-] (12) edge (13);
    \path [-] (13) edge (6);
\end{tikzpicture}
\end{minipage}
}

\caption{Filter responses used in (a) GCN and (b) Scattering
channels when applied to a constant signal $\boldsymbol{x}$ over a
graph with two cyclic substuctures connected by a single edge
bottleneck.}

\label{fig: toy example copy}
\end{figure}

Let us now revisit the example given in the main paper (see Fig.~3 or a copy in Fig.~\ref{fig: toy example copy}). We consider a graph consisting of 14 nodes organized in 2 cycles ($v_1 \sim v_2 \cdots v_6 \sim v_1$ and $v_7 \sim v_8 \cdots v_{14} \sim v_7$ here) of different length (i.e., 6 and 8 here), which are connected with one single edge between any two nodes taken from different cycles ($v_4$ and $v_7$ here). As this is a specific case of Lemma~\ref{lem:supp}, the filter responses of the GCN filter $\boldsymbol{g}_\theta$ for a constant signal $\boldsymbol{x}$ would indeed not distinguish between cycles as discussed in Sec.~7 of the main paper (with the pattern shown in Figs.~3 and~\ref{fig: toy example copy} for $\theta=1$). On the other hand, the filter responses of $\boldsymbol{\Psi}_3 \boldsymbol{x}$ on a constant signal (e.g., $\boldsymbol{x}=\ones_{14}\in \mathbb{R}^{14}$) on this graph can be verified empirically as follows:
%
% GCN: [1.         1.         0.9553418  1.07735027 0.9553418  1. 1.07735027 0.9553418  1.         1.         1.         1. 1.         0.9553418 ]
%
% Scattering [ 0.03302739  0.02029378 -0.01071427 -0.05230846 -0.01071427  0.02029378 -0.05350995 -0.01389074  0.0112315   0.01975119  0.01944811  0.01975119 0.0112315  -0.01389074]
\begin{center}
% \adjustbox{width=\textwidth}{
% \begin{tabular}{|r|rrrrrr|rrrrrrrr|}
% \multicolumn{1}{r}{~} & \multicolumn{6}{c}{$\overbrace{\hspace{180pt}}^\text{Cycle 1:}$} & \multicolumn{8}{c}{$\overbrace{\hspace{220pt}}^\text{Cycle 2:}$}\\
% \cline{2-15}
% \multicolumn{1}{r|}{~} & \multicolumn{1}{c}{$v_1$} & \multicolumn{1}{c}{$v_2$} & \multicolumn{1}{c}{$v_3$} & \multicolumn{1}{c}{$v_4$} & \multicolumn{1}{c}{$v_5$} & \multicolumn{1}{c|}{$v_6$} & \multicolumn{1}{c}{$v_7$} & \multicolumn{1}{c}{$v_8$} & \multicolumn{1}{c}{$v_9$} & \multicolumn{1}{c}{$v_{10}$} & \multicolumn{1}{c}{$v_{11}$} & \multicolumn{1}{c}{$v_{12}$} & \multicolumn{1}{c}{$v_{13}$} & \multicolumn{1}{c|}{$v_{14}$} \\
% \hline
% $\boldsymbol{A} \boldsymbol{x}$ \scriptsize($10^{-3}\times$) & \scriptsize 1,000.0 & \scriptsize 1,005.9 & \scriptsize 981.5 & \scriptsize 1,023.3 & \scriptsize 981.5 & \scriptsize 1,005.9 & \scriptsize 1,023.3 & \scriptsize 981.5 & \scriptsize 1,005.9 & \scriptsize 1,000.0 & \scriptsize 1,000.0 & \scriptsize 1,000.0 & \scriptsize 1,005.9 & \scriptsize 981.5 \\
% \cline{2-15}
% $\boldsymbol{\Psi}_3 \boldsymbol{x}$ \scriptsize($10^{-3}\times$)& \scriptsize 33.0 & \scriptsize 20.3 & \scriptsize -10.7 & \scriptsize -52.3 & \scriptsize -10.7 & \scriptsize 20.3 & \scriptsize -53.5 & \scriptsize -13.9 & \scriptsize 11.2 & \scriptsize 19.8 & \scriptsize 19.4 & \scriptsize 19.8 & \scriptsize 11.2 & \scriptsize -13.9 \\
% \hline
% \end{tabular}
% }

\begin{tabular}{|r|rrrrrr|}
% \multicolumn{1}{r}{~} & \multicolumn{6}{c}{$\overbrace{\hspace{120pt}}^\text{Cycle 1:}$} \\
\multicolumn{1}{r}{~} & \multicolumn{1}{c}{$v_1$} & \multicolumn{1}{c}{$v_2$} & \multicolumn{1}{c}{$v_3$} & \multicolumn{1}{c}{$v_4$} & \multicolumn{1}{c}{$v_5$} & \multicolumn{1}{c}{$v_6$} \\
\hline
$\boldsymbol{\Psi}_3 \boldsymbol{x}$ \scriptsize($10^{-3}\times$)& \scriptsize 33.0 & \scriptsize 20.3 & \scriptsize -10.7 & \scriptsize -52.3 & \scriptsize -10.7 & \scriptsize 20.3 \\
\hline
\end{tabular}

\begin{tabular}{|r|rrrrrrrr|}
% \multicolumn{1}{r}{~} & \multicolumn{8}{c}{$\overbrace{\hspace{120pt}}^\text{Cycle 2:}$}\\
\multicolumn{1}{r}{~} & \multicolumn{1}{c}{$v_7$} & \multicolumn{1}{c}{$v_8$} & \multicolumn{1}{c}{$v_9$} & \multicolumn{1}{c}{$v_{10}$} & \multicolumn{1}{c}{$v_{11}$} & \multicolumn{1}{c}{$v_{12}$} & \multicolumn{1}{c}{$v_{13}$} & \multicolumn{1}{c}{$v_{14}$} \\
\hline
$\boldsymbol{\Psi}_3 \boldsymbol{x}$ \scriptsize($10^{-3}\times$) & \scriptsize -53.5 & \scriptsize -13.9 & \scriptsize 11.2 & \scriptsize 19.8 & \scriptsize 19.4 & \scriptsize 19.8 & \scriptsize 11.2 & \scriptsize -13.9 \\
\hline
\end{tabular}

\end{center}
These responses with appropriate color coding give the illustration in Fig.~2 in the main paper. A similar distinction between cycles with band-pass filters can also be empirically verified for other cases covered by Lemma~\ref{lem:supp}. We leave further theoretical studies of this property of graph wavelets to future work.

% which clearly demonstrates a case where the GCN filter fails to distinguish between nodes in different cycles (regardless of downstream learned weights or biases) when the signal contains no geometric information. In contrast, after the operation $\boldsymbol{\Psi}_3 \boldsymbol{x}$, no node features are shared across cyclic substructures of different length, and therefore appropriate downstream operations can distinguish between them. We note that this example generalizes to arbitrary choices of the cycle lengths as long as the minimum length is strictly larger than 3, as well as structures with more than two cyclic substructures.

\section{Further Discussion of Related Work} % \label{sect_rel. work}

As many applied fields such as Bioinformatics and Neuroscience heavily rely on the analysis of graph-structured data, the study of reliable classification methods has received much attention lately.
In this work, we focus on the particular class of semi-supervised classification tasks, where GCN models \cite{kipf2016semi,li2018deeper} recently proved to be effective. Their theoretical studies reveal however that graph convolutions can be interpreted as Laplacian smoothing operations, which poses fundamental limitations on the approach.
Another branch of GNNs, manifested in \cite{velivckovic2017graph}, introduces self-attention mechanisms to determine adequate node-wise neighborhoods, which in turn alleviate  the mentioned shortcomings of GCN approaches.
Further,~\cite{nt2019revisiting} developed a theoretical framework based on graph signal processing, relying on the relation between frequency and feature noise, to show that GNNs perform a low-pass filtering on the feature vectors. In~\cite{abu2019mixhop}, multiple powers of the adjacency matrix were used to learn the higher-order neighborhood information, while~\cite{liao2019lanczosnet} used Lanczos algorithm to construct a low-rank approximation of the graph Laplacian that efficiently gathers multiscale information, demonstrated on citation networks and the QM8 quantum chemistry dataset.
Finally,~\citep{xu2019graph} studied  wavelets on graphs and collected higher-order neighborhood information based on wavelet transformation.

Together with the study of learned networks, recent studies have also introduced the construction of geometric scattering transforms, relying on manually crafted families of graph wavelet transforms~\citep{gama2019diffusion,gao2019geometric,zou2019graph}.
Similar to the initial motivation of geometric deep learning to generalize convolutional neural networks, the geometric scattering framework generalizes the construction of Euclidean scattering from~\cite{mallat2012group} to the graph setting.
Theoretical studies~\citep[e.g.,][]{gama2019diffusion,gama2019stability,perlmutter2019understanding} established energy preservation properties and the stability of these generalized scattering transforms to perturbations and deformations of graphs and signals on them. 
Moreover, the practical application of geometric scattering to whole-graph data analysis was studied in~\cite{gao2019geometric}, achieving strong classification results on social networks and biochemistry data, which established the effectiveness of this approach.

As discussed in the main paper, this work aims to combine the complementary strengths of GCN models and geometric scattering and to provide a new avenue for incorporating richer notions of regularity in GNNs. Further, our construction integrates trained task-driven components in geometric scattering architectures. Finally, while most previous work on geometric scattering focused on whole-graph settings, we consider node-level processing, which requires new considerations about the construction.

\section{Technical Details}

Similar to other neural networks, the presented Scattering GCN poses several architecture choices and hyperparameters that can be tuned and affect its performance. For simplicity, we set the last layer before the output classification to be the residual convolution layer and only consider one or two hybrid layers before it, each consisting of three GCN channels and two scattering channels. We note that this restricted setup simplifies the network tuning process and was sufficient in our experiments to obtain promising results (outperforming other methods, as shown in the main paper), but can naturally be generalized further to deeper or wider architectures in practice. Furthermore, based on preliminary results, \emph{Cora}, \emph{Citeseer} and \emph{Pubmed} were set to use only one hybrid layer as the addition of a second one was not cost-effective (considering the added complexity of a grid search for tuning hyperparameters based on validation results). For \emph{DBLP}, two layers were used due to a significant increase in performance. We note that even with a single hybrid layer our model achieves $73.1\%$ test accuracy (compared to the reported $81.5\%$ for two layers) and still significantly outperforms GAT ($66.1\%$) and the other methods (below $60\%$).

% \begin{figure}[ht]
% \vskip 0.01in
% \begin{center}
% \centerline{\includegraphics[width=\columnwidth]{Figures/channels.pdf}}
% \caption{Scattering GCN performance on different datasets with different parameters. a) Cora, b) Citeseer, c) Pubmed, d) DBLP}
% \label{fig:channels}
% \end{center}
% \vskip -0.3in
% \end{figure}

% \textcolor{red}{Hyperparameters to tune}
% \paragraph{Hyperparameters}
% \begin{enumerate}
% \item \textit{Cora, Pubmed and Citeseer}: The width of the low-passing channels were fixed and we performed grid searching on the scattering channels, scattering channels' width, the scattering order moments $q={1,2,3,4}$ and the graph residual convolution parameter $\alpha$. 
% \item \textit{DBLP}: As mentioned earlier in the manuscript, the edges-to-nodes ratio indicates DBLP has rich connectivity, which exacerbates the oversmoothing problem in graph signal processing. For overcoming this problem, we use two scattering layers in our model.
% \end{enumerate}
% Fig.~\ref{fig:channels} shows the structures and parameters on different datasets. For Cora, Pubmed and Citeseer, the scattering order moment $q$ is 4, for DBLP $q$ = 1. The residual convolution parameter $\alpha$ in selected from 0.01,0.1,0.5 and 1.0. 
% In our experiments, we also observed that more detailed hyperparameters searching would contribute to better performance. For example, we found that if $\alpha$ is set to 0.35, the performance on Cora increases to 84.2\%. 
% The optimal parameters are selected according to performance on the validation data.

\paragraph{Validation \& testing procedure:} All tests were done using train-validation-test splits of the datasets, where validation accuracy is used for tuning hyperparameters and test accuracy is reported in the comparison table. The same splits were used for all methods for a fair comparison. To ensure our evaluation is comparable with previous work, for \emph{Citeseer}, \emph{Cora} and \emph{Pubmed} we used the same settings as in~\cite{kipf2016semi}, following the standard practice used in other work reporting results on these datasets. For \emph{DBLP}, as far as we know, no common standard is established in the literature. Here, we used a ratio of $5:1:1$ between train, validation, and test.

\paragraph{Hyperparameter tuning:} Given the general network architectures and train-validation-test splits, the hyperparameter tuning was performed for each dataset using grid search guided by validation accuracy. The grid covered the tuning of the residual convolution via $\alpha$, the nonlinarity exponent $q$ (inspired by scattering moments), the scattering channel configuration (i.e., scales used in these two channels), and the widths of channels in the network. The results of this tuning process are presented in the following table.
\begin{center}
\begin{tabular}{|c||c|c|cc|ccccc|}
\multicolumn{3}{c}{~} & \multicolumn{2}{c}{$\overbrace{\hspace{60pt}}^\text{Scat.\ config.:}$} & \multicolumn{5}{c}{$\overbrace{\hspace{130pt}}^\text{Channel widths:}$}\\
\cline{2-10}
\multicolumn{1}{c}{~} & \multicolumn{1}{|c}{$\alpha$} & \multicolumn{1}{|c}{$q$} & \multicolumn{1}{|c}{$\boldsymbol{U}_{J_1}$} & \multicolumn{1}{c}{$\boldsymbol{U}_{J_2}$} & \multicolumn{1}{|c}{$\,\boldsymbol{A}^1\,$} & \multicolumn{1}{c}{$\,\boldsymbol{A}^2\,$} & \multicolumn{1}{c}{$\,\boldsymbol{A}^3\,$} & \multicolumn{1}{c}{$\boldsymbol{U}_{p_1}$} & \multicolumn{1}{c|}{$\boldsymbol{U}_{p_2}$} \\
\hline
\emph{Citeseer} & 0.50 & 4 & $\boldsymbol{\Psi}_2$& $\boldsymbol{\Psi}_2|\boldsymbol{\Psi}_3|$  & 10 & 10 & 10 & 9 & 30 \\
\emph{Cora} & 0.35 & 4 & $\boldsymbol{\Psi}_1 $& $\boldsymbol{\Psi}_3$ & 10 & 10 & 10 & 11 & 6\\
\emph{Pubmed} & 1.00 & 4 & $\boldsymbol{\Psi}_1 $& $\boldsymbol{\Psi}_2$ & 10 & 10 & 10 & 13 & 14 \\
\emph{DBLP (1st layer)} & 1.00 & 4 & $\boldsymbol{\Psi}_1 $& $\boldsymbol{\Psi}_2$ & 10 & 10 & 10 & 30 & 30 \\
\emph{DBLP (2nd layer)} & 0.10 & 1 & $\boldsymbol{\Psi}_1 $  & $\boldsymbol{\Psi}_2 $  & 40 & 20 & 20 & 20 & 20 \\
\hline
\end{tabular}
\end{center}
It should be noted that for \emph{DBLP}, the hybrid-layer parameters are shared between the two used layers in order to simplify the tuning process, which was generally less exhaustive than for the other three datasets, since even with limited tuning our method significantly outperformed all other methods. That being said, we note that the difference in effectiveness of architecture and hyperparameter choices (as well as the increased performance of our approach compared to others) observed in this case could be a result of the significantly different connectivity exhibited by its graph as discussed briefly in Sec.~8 (of the main paper). Regardless, as more exhaustive tuning would not degrade (and likely improve) the results obtained from Scattering GCN, we view our limited tuning done here as sufficient for establishing the advantages provided by our approach over other methods and leave a more intensive study of the \emph{DBLP} dataset to future work.

\paragraph{Hardware \& software environment:}
All comparisons were executed on the same HPC cluster with intel i7-6850K CPU and NVIDIA TITAN X Pascal GPU. Scattering GCN was implemented in Python using the PyTorch \cite{paszke2019pytorch} framework. Implementations of all other methods were taken directly from the code accompanying their publications.

\section{Ablation Study}
The two main components of our Scattering GCN architecture contribute together to achieve significant improvements over pure GCN models.
Namely, these are the additional scattering channels (i.e., $\boldsymbol{U}_{J_1}$ and $\boldsymbol{U}_{J_2}$) and the residual convolution (controlled by the hyperparemeter $\alpha$). To further explore their contribution and the hyperparameter space for their tuning, Tab.~\ref{table:cora1}-\ref{table:cora5} show classification results over the Cora dataset for $\alpha=0.01,0.1,0.35,0.5,1.0$ (controlling the residual convolution layer) over multiple scattering channel configurations. For presentation brevity and simplicity, we focus our presented ablation benchmark on this dataset here, but note that similar results are also observed on the other datasets. The rows and columns in each table denote the two scattering channels used in the Scattering GCN, together with the three GCN channels (i.e., for $\boldsymbol{A}^k$, $k=1,2,3$).

\begin{table}[!b]
\begin{minipage}{0.48\textwidth}
\caption{Classification accuracies on Cora with $\alpha = 0.01$ with average accuracy 80.4\% over all scales. }
\label{table:cora1}
\centering
\begin{small}
\begin{sc}
\adjustbox{width=\textwidth}{
\begin{tabular}{|c|c|c|c|c|c|}
\hline
Accuracy & $\boldsymbol{\Psi}_1 $    & $\boldsymbol{\Psi}_2$     & $\boldsymbol{\Psi}_3 $    & $\boldsymbol{\Psi}_2|\boldsymbol{\Psi}_3|$     & $\boldsymbol{\Psi}_1|\boldsymbol{\Psi}_2|$     \\ \hline
$\boldsymbol{\Psi}_1 $         & 0.808 & 0.808 & 0.805 & 0.806 & 0.806 \\ \hline
$\boldsymbol{\Psi}_2 $         & 0.809 & 0.809 & 0.806 & 0.806 & 0.806 \\ \hline
$\boldsymbol{\Psi}_3 $         &  0.802& 0.804 & 0.801 & 0.801 & 0.800 \\ \hline
$\boldsymbol{\Psi}_2|\boldsymbol{\Psi}_3|$         &   0.802   &  0.804 & 0.801 & 0.800 & 0.799 \\ \hline
$\boldsymbol{\Psi}_1|\boldsymbol{\Psi}_2|$         &   0.802  &   0.804    &  0.801     &   0.800   & 0.800  \\ \hline
\end{tabular}
}
\end{sc}
\end{small}
\end{minipage}
\hfill
\begin{minipage}{0.48\textwidth}
\caption{Classification accuracies on Cora with $\alpha = 0.1$ with average accuracy 80.9\% over all scales. }
\label{table:cora2}
\centering
\begin{small}
\begin{sc}
\adjustbox{width=\textwidth}{
\begin{tabular}{|c|c|c|c|c|c|}
\hline
Accuracy & $\boldsymbol{\Psi}_1 $    & $\boldsymbol{\Psi}_2$     & $\boldsymbol{\Psi}_3 $    & $\boldsymbol{\Psi}_2|\boldsymbol{\Psi}_3|$     & $\boldsymbol{\Psi}_1|\boldsymbol{\Psi}_2|$     \\ \hline
$\boldsymbol{\Psi}_1 $         & 0.813& 0.813 & 0.812 & 0.808 & 0.809 \\ \hline
$\boldsymbol{\Psi}_2 $         & 0.817 & 0.817 & 0.810& 0.810 & 0.815 \\ \hline
$\boldsymbol{\Psi}_3 $         & 0.810& 0.808& 0.801 & 0.800 & 0.806 \\ \hline
$\boldsymbol{\Psi}_2|\boldsymbol{\Psi}_3|$         &   0.812   & 0.811& 0.801 & 0.800 & 0.809\\ \hline
$\boldsymbol{\Psi}_1|\boldsymbol{\Psi}_2|$         &  0.813  &   0.811    &  0.802     &   0.802   & 0.809  \\ \hline
\end{tabular}
}
\end{sc}
\end{small}
\end{minipage}

\begin{minipage}{0.48\textwidth}
\caption{Classification accuracies on Cora with $\alpha = 0.35$ with average accuracy 83.5\%.}
\label{table:cora3}
\begin{center}
\begin{small}
\begin{sc}
\adjustbox{width=\textwidth}{
\begin{tabular}{|c|c|c|c|c|c|}
\hline
Accuracy & $\boldsymbol{\Psi}_1 $    & $\boldsymbol{\Psi}_2$     & $\boldsymbol{\Psi}_3 $    & $\boldsymbol{\Psi}_2|\boldsymbol{\Psi}_3|$     & $\boldsymbol{\Psi}_1|\boldsymbol{\Psi}_2|$     \\ \hline
$\boldsymbol{\Psi}_1 $         & 0.838 & 0.836 & 0.835 & 0.837 & 0.837 \\ \hline
$\boldsymbol{\Psi}_2 $         & {\bf 0.842} & 0.836 & 0.835 & 0.838 & 0.837 \\ \hline
$\boldsymbol{\Psi}_3 $         &  0.835     &   0.836    & 0.833 & 0.833  & 0.832  \\ \hline
$\boldsymbol{\Psi}_2|\boldsymbol{\Psi}_3|$  &  0.835     &   0.836   &    0.833   & 0.833 & 0.831 \\ \hline
$\boldsymbol{\Psi}_1|\boldsymbol{\Psi}_2|$         &  0.835     &  0.836   &    0.832  &   0.833    & 0.831 \\ \hline
\end{tabular}
}
\end{sc}
\end{small}
\end{center}
\end{minipage}
\hfill
\centering
\begin{minipage}{0.48\textwidth}
\caption{Classification accuracies on Cora  with $\alpha = 0.5$ with average accuracy 82.7\% over all scales.}
\label{table:cora4}
\centering
\begin{small}
\begin{sc}
\adjustbox{width=\textwidth}{
\begin{tabular}{|c|c|c|c|c|c|}
\hline
Accuracy & $\boldsymbol{\Psi}_1 $    & $\boldsymbol{\Psi}_2$     & $\boldsymbol{\Psi}_3 $    & $\boldsymbol{\Psi}_2|\boldsymbol{\Psi}_3|$     & $\boldsymbol{\Psi}_1|\boldsymbol{\Psi}_2|$     \\ \hline
$\boldsymbol{\Psi}_1 $       & 0.828  & 0.828 & 0.836 & 0.836 &0.835\\ \hline
$\boldsymbol{\Psi}_2$        &  0.828&0.833& 0.830 & 0.830 &0.827\\ \hline
$\boldsymbol{\Psi}_3 $        & 0.821& 0.829 &0.826 & 0.826 & 0.826 \\ \hline
$\boldsymbol{\Psi}_2|\boldsymbol{\Psi}_3|$        & 0.820& 0.828  & 0.826 & 0.825 &0.825 \\ \hline
$\boldsymbol{\Psi}_1|\boldsymbol{\Psi}_2|$        & 0.821 & 0.829  &0.824  & 0.824 & 0.824\\ \hline
\end{tabular}
}
\end{sc}
\end{small}
\end{minipage}

\begin{minipage}{0.48\textwidth}
\caption{Classification accuracies on Cora  with $\alpha = 1.0$ with average accuracy 82.3\% over all scales.}
\label{table:cora5}
\centering
\begin{small}
\begin{sc}
\adjustbox{width=\textwidth}{
\begin{tabular}{|c|c|c|c|c|c|}
\hline
Accuracy & $\boldsymbol{\Psi}_1 $    & $\boldsymbol{\Psi}_2$     & $\boldsymbol{\Psi}_3 $    & $\boldsymbol{\Psi}_2|\boldsymbol{\Psi}_3|$     & $\boldsymbol{\Psi}_1|\boldsymbol{\Psi}_2|$     \\ \hline
$\boldsymbol{\Psi}_1 $       & 0.817  & 0.818  & 0.827 & 0.827 & 0.824\\ \hline
$\boldsymbol{\Psi}_2$        &  0.820 & 0.819& 0.824 & 0.824 & 0.823 \\ \hline
$\boldsymbol{\Psi}_3 $        & 0.826 & 0.823 &0.823 & 0.823 & 0.821 \\ \hline
$\boldsymbol{\Psi}_2|\boldsymbol{\Psi}_3|$        & 0.825& 0.823  & 0.823 & 0.823 & 0.821 \\ \hline
$\boldsymbol{\Psi}_1|\boldsymbol{\Psi}_2|$        & 0.822 & 0.823  & 0.823  & 0.823 & 0.821\\ \hline
\end{tabular}
}
\end{sc}
\end{small}
\end{minipage}
\end{table}

First, we consider the importance of the residual graph convolution layer. We note that setting $\alpha = 0$ effectively ignores this layer (i.e., by setting its operation to be the identity), while increasing $\alpha$ makes the filtering provided by it to be more dominant until $\alpha=1$, where it essentially becomes a random walk-based low-pass filter. Therefore, to evaluate the importance of this component of our architecture, it is sufficient to evaluate the impact of $\alpha$ on the classification accuracy. Indeed, our results (see Tab.~\ref{table:cora1}-\ref{table:cora5}) indicate that increasing $\alpha$ to non-negligible nonzero values improves classification performance, which we interpret to be due to the removal of high-frequency noise. However, when $\alpha$ further increases (in particular when $\alpha=1$ in this case) the smoothing provided by this layer degrades the performance to a level close to the traditional GCN~\citep{kipf2016semi}. Therefore, these results suggest that when well tuned (e.g., as done via grid search in this work), the graph residual convolution plays a critical role in improving results, which can also be seen by the results of the hyperparameter tuning shown in the previous section.

Next, we consider the scales used in the two scattering channels of our hybrid architecture, which correspond to the rows and columns of Tab.~\ref{table:cora1}-\ref{table:cora5} here. While our results show that the network is relatively robust to this choice, we can observe that generally utilizing purely second-order coefficients gives slightly worse results than either first-order ones or a mix of first- and second-order coefficients. Nevertheless, most configurations of scattering scales (with appropriate choice of $\alpha$) give better results than pure GCN, thus indicating that added information is extracted by scattering channels.

We note that for $\alpha = 0.35$ (Tab.~\ref{table:cora3}), all scale configurations outperform GAT (83\%) and all other reported methods in Tab.~2 (of the main paper). As a result, even the average accuracy over all scale configurations (83.5\%) in this case shows an improvement over these other methods, thus further establishing the advantage of our approach. It is important to mention that while this improvement is affected by the tuning of the residual convolution, it also relies on the addition of scattering channels. Indeed, by itself, the residual convolution layer only applies a low-pass (smoothing) filter and therefore, without scattering channels, would essentially be equivalent to a conventional GCN.

We remark that the results presented in this work are based on a limited grid search, while the ablation study here indicates that many of the possible configurations provide noticeable classification performance improvement over other methods. It is likely that the reported results, in fact, provide a lower bound on the improvement attained by Scattering GCN, while a more exhaustive optimization of the architecture and its hyperparameters can further improve and solidify its advantages. We leave such exhaustive study for future work, which will also consider the incorporation of other advanced components (e.g., attention mechanisms) in the model architecture.

Finally, to further validate the importance of band-pass information added by the presented architecture here, we provide an ablation study of the impact each channel has on classification performance. We focus here on $\alpha=0.35$ with the best configuration on Cora (i.e., achieving 84.2\% accuracy when all channels are used). Then, we remove each of the band-pass or low-pass channels individually from the network, and reevaluate network performance with the remaining four channels. Our results, presented in Tab.~\ref{table:ablation_channnels}, indicate that while information captured by $\boldsymbol{A}$,$\boldsymbol{A}^2$ and $\boldsymbol{A}^3$ is important for the classification task, which is to be expected given the prevalence of such filters in GCNs, the band-pass information extracted by the scattering channels (with $\boldsymbol{\Psi}_2$ and $\boldsymbol{\Psi}_1$ in this case) plays a crucial role in achieving the performance of our method. In particular, we note that $\boldsymbol{\Psi}_1$ in this case has a major impact on the accuracy, driving the difference between underperforming and outperforming GAT, thus strengthening the claim that important information in graph features is contained in higher frequencies extracted by band-pass filtering, which is not recovered by smoothing operations.

\begin{table}[!h]
\centering

\begin{sc}

\caption{Impact of removing each individual channel from the optimal configuration on Cora, while classifying using the remaining four channels. Full Scattering GCN accuracy provided for reference.}
\label{table:ablation_channnels}
\begin{tabular}{|c|c|c|c|c|c|c|c|}
\hhline{------|~|-}
Removed Channel &  $\boldsymbol{A} $    & $\boldsymbol{A}^2$     & $\boldsymbol{A}^3 $    & $\boldsymbol{\Psi}_2$     & $\boldsymbol{\Psi}_1$ & $\,$ & Scattering GCN \\
\hhline{------|~|-}
Accuracy & 82.0 & 80.7 & 80.9 & 83.7 & 82.7 & $\,$ & 84.2\\
\hhline{------|~|-}
\end{tabular}
\end{sc}

\end{table}

\section{Implementation}

Python code accompanying this work is available on \href{http://github.com/dms-net/scatteringGCN}{github.com/dms-net/scatteringGCN}.

\end{document}